\documentclass{article}
\usepackage{arxiv}

\usepackage[utf8]{inputenc} 
\usepackage[T1]{fontenc}    
\usepackage{hyperref}       
\usepackage{url}            
\usepackage{booktabs}       
\usepackage{amsfonts}       
\usepackage{nicefrac}       
\usepackage{microtype}      
\usepackage{lipsum}
\usepackage{fancyhdr}       
\usepackage{graphicx}       
\graphicspath{{media/}}     
\usepackage{xspace}

\usepackage{natbib} 
\usepackage{subcaption}
\usepackage[normalem]{ulem} 
\usepackage{tcolorbox}
\usepackage{booktabs} 
\usepackage{mathtools} 
\usepackage{stmaryrd} 
\mathtoolsset{showonlyrefs}
\usepackage{amsthm} 
\usepackage{xcolor}
\usepackage{enumitem}
\usepackage{tikz}
\usepackage{balance}

\usepackage{algorithm}
\usepackage{algorithmic}
\usepackage{natbib}
\usepackage{eso-pic} 
\RequirePackage{forloop}
\RequirePackage{url}

\usepackage[symbol]{footmisc}
\renewcommand{\thefootnote}{\fnsymbol{footnote}}

\newtheorem{theorem}{Theorem}

\newtheorem{lemma}{Lemma}
\newtheorem{definition}{Definition}
\newtheorem{assumption}{Assumption}
\newtheorem{property}{Property}

\newcommand{\NSF}{\texttt{NSF}}
\newcommand{\algname}{\texttt{ELF}\xspace}

\newcommand{\pitheta}{{\pi_\theta}}
\newcommand{\E}{{\mathbf E}}

\makeatletter
\newcommand{\oset}[3][0ex]{%
  \mathrel{\max\limits^{
    \vbox to#1{\kern-2\ex@
    \hbox{$\scriptstyle\max$}\vss}}}}

\makeatother

\definecolor{plotColor1}{HTML}{E6194B}
\definecolor{plotColor2}{HTML}{F58230}
\definecolor{plotColor3}{HTML}{61B64F}
\definecolor{plotColor4}{HTML}{5CB5FF}
\definecolor{plotColor5}{HTML}{E799FF}

\DeclareRobustCommand\elfline  {\tikz[baseline=-0.6ex]\draw[plotColor1, thick] (0,0)--(0.6,0);}
\DeclareRobustCommand\lrline  {\tikz[baseline=-0.6ex]\draw[plotColor2, thick] (0,0)--(0.6,0);}
\DeclareRobustCommand\qsaDPline  {\tikz[baseline=-0.6ex]\draw[plotColor3, thick, dashed] (0,0)--(0.6,0);}
\DeclareRobustCommand\qsaEqOddsline  {\tikz[baseline=-0.6ex]\draw[plotColor3, thick, dash dot] (0,0)--(0.6,0);}
\DeclareRobustCommand\qsaEqOppline  {\tikz[baseline=-0.6ex]\draw[plotColor3, thick, dash pattern={on 3pt off 1pt on 1pt off 1pt on 1pt off 1pt}] (0,0)--(0.6,0);}
\DeclareRobustCommand\qsaPEline  {\tikz[baseline=-0.6ex]\draw[plotColor3, thick] (0,0)--(0.6,0);}
\DeclareRobustCommand\qsaDisImpline  {\tikz[baseline=-0.6ex]\draw[plotColor3, thick, dotted] (0,0)--(0.7,0);}
\DeclareRobustCommand\flDPline  {\tikz[baseline=-0.6ex]\draw[plotColor4, thick] (0,0)--(0.6,0);}
\DeclareRobustCommand\flEqOddsline  {\tikz[baseline=-0.6ex]\draw[plotColor4, thick, dashed] (0,0)--(0.7,0);}
\DeclareRobustCommand\fcline  {\tikz[baseline=-0.6ex]\draw[plotColor5, thick] (0,0)--(0.6,0);}

\title{Enforcing Delayed-Impact Fairness Guarantees
}

\author{
  Aline Weber$^*$, Blossom Metevier$^*$, Yuriy Brun, Philip S.\ Thomas, Bruno Castro da Silva\\
  Manning College of Information and Computer Sciences \\
  University of Massachusetts \\
}

\begin{document}
\maketitle
\def\thefootnote{*}\footnotetext{These authors contributed equally to this work.}\def\thefootnote{\arabic{footnote}}

\begin{abstract}
Recent research has shown that seemingly fair machine learning models, when used to inform decisions that have an impact on peoples' lives or well-being (e.g., applications involving education, employment, and lending), can inadvertently increase social inequality in the long term. 
This is because prior fairness-aware algorithms only consider static fairness constraints, such as equal opportunity or demographic parity. However, enforcing constraints of this type may result in models that have negative long-term impact on disadvantaged individuals and communities. 
We introduce \algname (Enforcing Long-term Fairness), the first classification algorithm that provides high-confidence fairness guarantees in terms of long-term, or delayed, impact. 
We prove that the probability that \algname returns an unfair solution is less than a user-specified tolerance
and that (under mild assumptions), given sufficient training data, \algname is able to find and return a fair solution if one exists. We
show experimentally that our algorithm can successfully mitigate 
long-term unfairness.

\end{abstract}


\section{Introduction}
\label{sec: introduction}

Using machine learning (ML) for high-stakes applications, such as lending, hiring, and criminal sentencing, may potentially harm historically disadvantaged communities~\citep{flage2018ethnic, blass2019algorithmic, bartlett2021consumer}.  
For example, software meant to guide lending decisions has been shown to exhibit racial bias~\citep{bartlett2021consumer}. 
Extensive research has been devoted to  algorithmic approaches that promote fairness and ameliorate concerns of bias and discrimination for socially impactful applications. 
The bulk of this research has focused on the classification setting, in which an ML model
must make predictions given information about a person or community.

Most fairness definitions studied in the classification setting are \emph{static}~\citep{liu2018delayed} in that they do not consider how a classifier's predictions impact the long-term well-being of a community.
%
In their seminal paper, \citet{liu2018delayed}~
%
show that classifiers' predictions that appear fair with respect to static fairness criteria can nevertheless negatively impact the long-term wellness of the community it aims to protect. 
Importantly, 
they assume that the precise analytical relationship between a classifier's prediction and its long-term impact, or \emph{delayed impact} (DI), is known. This is also the case in other fairness-related work \citep{d2020fairness, ge2021towards, hu2022achieving}.
\emph{Designing classification algorithms that mitigate negative delayed impact when this relationship is not known has remained an open problem.}

In this work, we introduce \algname (Enforcing Long-term Fairness), the first classification algorithm that solves this open problem. 
\algname does not require access to 
an analytic model of the delayed impact of a classifier's predictions. Instead, it works under the less strict assumption that the algorithm has access to historical data 
containing observations of the delayed impact that results from the predictions of an existing classifier. We illustrate this setting below with an example.

\noindent\textbf{Loan repayment example.} As a running example, consider 
a bank that wishes to increase its profit by maximizing successful loan repayments. 
The bank's decisions are informed by a classifier that predicts repayment success.
These decisions may affect the long-term financial well-being of loan applicants, such as their savings rate or debt-to-income ratio two years after a lending decision is made.
Taking this delayed impact into account is important: when a subset of the population is disadvantaged, the bank may want (or be required by law) to maximize profit subject to a fairness constraint that considers the disadvantaged group's long-term well-being. 
Unfortunately, existing methods that address this problem can only be used if analytical models of how repayment predictions affect long-term financial well-being are available. 
Constructing such models is challenging: many complex factors influence how different demographic groups in a \mbox{given community are affected by financial decisions.}
%

%
\algname, by contrast, can ensure delayed-impact fairness with high confidence as long as 
the bank can collect data about the long-term financial well-being of loan applicants, following decisions based on an existing classifier. 
As an example, the bank might access information about the savings rate of an applicant two years after a lending decision is made. 
Here, delayed impact could be defined as the real-valued savings rate. 
However, we emphasize that our approach works with any metric of delayed impact that can be observed and quantified, including more holistic metrics than savings rate.
%
%
%
%
As one motivating use case, this work provides the algorithmic tools to responsibly apply ML for this task.\footnote{Notice that if used by adversaries, our method could be used to enforce bias instead of minimizing it.}


\noindent \textbf{Contributions.} We present \algname, the first method capable of enforcing DI fairness when the analytical relationship between predictions and DI is not known \emph{a priori}. 
To accomplish this, we simultaneously formulate the fair classification problem as both a classification and a reinforcement learning problem---classification for optimizing the primary objective (a measure of classification loss) and reinforcement learning when considering DI.
We prove that \textbf{1)} the probability that \algname returns a model that is unfair (in terms of DI) is at most $\delta$, where $\delta$ is a hyperparameter that can be set appropriately for the application at hand; and \textbf{2)} given sufficient training data, \algname is able to find and return a solution that is fair if one exists.
We provide an empirical analysis of \algname's performance while varying both the amount of training data and the influence that a classifier's predictions have on DI. 

\noindent \textbf{Limitations and future work.} \algname's high probability fairness guarantees only hold if the world has not changed between the time training data was collected and the time the trained classifier is deployed. While this (stationarity) assumption is common in ML, it may be unnatural in this setting since gathering data that includes measures of long-term impact requires that a correspondingly long duration of time has passed, and so nonstationarity of the data-generating distribution could compound over time to make this assumption unreasonable.
For example, the way a group of people was affected by model predictions a decade ago may not be reflective of the present.
While providing guarantees when nonstationarity occurs is important future work, in this paper we focus on the important first step 
of 
%
providing the first classification algorithm that provides DI fairness guarantees in the stationary setting.

Additionally, notice that some applications may require long-term \emph{and} static fairness to be simultaneously satisfied.
Appendix~\ref{app: extensions} shows how \algname can be used in conjunction with prior methods~\citep{metevier2019offline} to enforce a broader set of fairness definitions that includes common static fairness definitions. 
We leave an empirical analysis of these settings to future work. \looseness=-1
 
\section{Problem statement}
\label{sec: problem statement}
We now formalize the problem of classification with delayed-impact fairness guarantees. 
As in the standard classification setting, a dataset consists of $n$ data points, the $i$\textsuperscript{th} of which contains $X_i$, a feature vector describing, e.g., a person, and a label $Y_i$. 
%
%
Each data point also contains a set of \emph{sensitive attributes}, such as race and gender.
Though our algorithm works with an arbitrary number of such attributes, for brevity our notation uses a single attribute, $T_i$. 

We assume that each data point also contains a prediction, $\widehat Y_i^\beta$, made by a stochastic classifier, or model, $\beta$.
We call $\beta$ the \emph{behavior model}, defined as
$\beta(x,\hat y) \coloneqq \Pr(\widehat Y_i^\beta {=} \hat y | X_i{=}x)$.
%
%
Let $I^\beta_i$ be a measure of the \emph{delayed impact} (DI) resulting from deploying  $\beta$ for the person described by the $i$\textsuperscript{th} data point. 
In our running example, $I^\beta_i$ corresponds to the savings rate two years after the prediction $\widehat Y_i^\beta$ was used to decide whether the $i$\textsuperscript{th} client should get a loan.
We assume that larger values of $I^\beta_i$ correspond to better DI. 
We append $I^\beta_i$ to each data point and thus define the dataset to be a sequence of $n$ independent and identically distributed (i.i.d.) data points $D\coloneqq\{(X_i, Y_i, T_i, \widehat Y^\beta_i, I^\beta_i )\}^n_{i=1}$. 
%
%
For notational clarity, when referring to an arbitrary data point, we write $X,Y,T, \widehat Y^\beta$, and $I^\beta$ without subscripts to denote $X_i, Y_i, T_i, \widehat Y^\beta_i$, and $I^\beta_i$, respectively. 

Given a dataset $D$, the goal is to construct a classification algorithm that takes as input $D$ and outputs a new model $\pitheta$ that is as accurate as possible while enforcing delayed-impact constraints.\footnote{Our algorithm works with arbitrary performance objectives, not just accuracy.} 
%
%
This new model $\pitheta$ is of the form $\pitheta(x, \hat y) \coloneqq \Pr(\widehat Y^\pitheta{=}\hat y|X{=}x)$, where $\pitheta$ is parameterized by a vector $\theta \in \Theta$ (e.g., the weights of a neural network), for some feasible set $\Theta$, and where $\widehat Y^\pitheta$ is the prediction made by $\pitheta$ given $X$. 
Like $I^\beta_i$, let $I^\pitheta_i$ be the delayed impact if the model outputs the prediction $\widehat Y^\pitheta_i$.

We consider the standard classification setting, in which a classifier's prediction depends only on the feature vector $X$ (see Assumption~\ref{ass: markov property}). Furthermore, we assume that regardless of the model used to make predictions, the distribution of DI given a prediction remains the same (see Assumption~\ref{ass: switch prediction}).\footnote{See Appendix \ref{app: assumption intuitions} for a discussion on the intuition and implications of these assumptions.}  
\begin{assumption}
\label{ass: markov property}
A model's prediction $\widehat Y^{\pitheta}$ is conditionally independent of $Y$ and $T$ given $X$. That is, for all $x,y,t,$ and $\hat y$, $\Pr(\widehat Y^\pitheta {=} \hat y | X{=}x, Y{=}y, T{=}t) = \Pr(\widehat Y^\pitheta {=} \hat y | X{=}x).$
%
\end{assumption}
%
%
\begin{assumption}
\label{ass: switch prediction}
For all $ x,y,t,\hat y$, and $i$,  \begin{equation}\Pr(I^{\beta}{=}i|X{=}x,Y{=}y,T{=}t,\widehat Y^{\beta}{=}\hat y)= 
\Pr(I^{\pi_{\theta}}{=}i|X{=}x,Y{=}y,T{=}t,\widehat Y^{\pi_{\theta}}{=}\hat y).
\end{equation}
\end{assumption}
Our problem setting can alternatively be described from the reinforcement learning (RL) perspective, where feature vectors are the states of a Markov decision process,  
%
%
predictions are the actions taken by an agent, and DI is the reward received after the agent takes an action (makes a prediction) given a state (feature vector). 
From this perspective, Assumption~\ref{ass: switch prediction} asserts that regardless of the policy (model) used to choose actions, the distribution of rewards given an action remains the same. 

We consider $k$ \emph{delayed-impact objectives} $g_j: \Theta \rightarrow \mathbb{R}, j \in \{1,...,k\}$, that take as input a parameterized model $\theta$ and return a real-valued measurement of fairness in terms of delayed impact. We adopt the convention that $g_j(\theta) \leq 0$ iff $\theta$ causes behavior that is fair with respect to delayed impact, and $g_j(\theta) > 0$ otherwise. To simplify notation, we assume there exists only a single DI objective (i.e., $k=1$) and later show how to enforce multiple DI objectives (see Algorithm~\ref{alg: main algorithm multiple constraints}).
We focus on the case in which each DI objective is based on a conditional expected value having the form 
\begin{equation}
    \label{eqn: delayed impact objective}
    g(\theta) \coloneqq \tau - \mathbf E[I^\pitheta | c(X,Y,T)],
\end{equation}
where $\tau \in \mathbb R$ is a tolerance and $c(X,Y,T)$ is a Boolean conditional relevant to defining the objective. Notice that this form of DI objective allows us to represent DI fairness notions studied in the literature such as~\citeauthor{liu2018delayed}'s long-term improvement.\footnote{Our method is not limited to this form. In Appendix~\ref{app: extensions}, we discuss how \algname can provide similar high-confidence guarantees for other forms of DI.}
To make~\eqref{eqn: delayed impact objective} more concrete, consider our running example where the bank would like to enforce a fairness definition that protects a disadvantaged group, $A$. In particular, the bank wishes to ensure that, for a model $\pitheta$ being considered, the future financial well-being---in terms of savings rate---of applicants in group $A$ (impacted by $\pitheta$'s repayment predictions) does not decline relative to those induced by the previously deployed model, $\beta$. 
In this case, $I^\pitheta$ is the financial well-being of an applicant $t$ months after the loan application and $c(X,Y,T)$ is the Boolean event indicating if an applicant is in group $A$. 
%
%
Lastly, $\tau$ could represent a threshold on which the bank would like to improve, e.g., the average financial well-being of type $A$ applicants given historical data collected using $\beta$: $\tau = \frac{1}{n_-}\sum_{d=1}^nI^\beta_d\llbracket T_d = A \rrbracket$, where $\llbracket \cdot \rrbracket$ denotes the Iverson bracket and $n_-=\sum_{d=1}^n\llbracket T_d = A \rrbracket $ is the number of applicants of type $A$. 
In this setting, the bank is interested in enforcing the following delayed-impact objective: $\mathbf E[I^\pitheta|T = A] \geq \tau$. 
Then, defining $g(\theta) = \tau - \mathbf E[I^\pitheta|T= A]$
ensures that $g(\theta) \leq 0$ iff the new model $\pitheta$ satisfies the DI objective.
Notice that an additional constraint of the same form can be added to protect other applicant groups.

\subsection{Algorithmic properties of interest}
As discussed above, we would like to ensure that $g(\theta) \leq 0$, since this delayed-impact objective implies that $\theta$ (the model returned by a classification algorithm) is fair with respect to DI.
However, this is often not possible, as it requires extensive prior knowledge of how predictions influence DI.
%
%
Instead, we aim to create an algorithm that uses historical data to reason about its confidence that $g(\theta) \leq 0$. 
%
That is, we wish to construct a classification algorithm, $a$, where $a(D)\in \Theta$ is the solution provided by $a$ when given dataset $D$ as input, that satisfies DI constraints of the form
\begin{equation}
    \label{eqn: delayed impact constraint}
    \Pr(g(a(D)) \leq 0) \geq 1-\delta,
\end{equation}
where $\delta \in (0, 1)$ limits the admissible probability that the algorithm returns a model that is unfair with respect to the DI objective. Algorithms that satisfy~\eqref{eqn: delayed impact constraint} are called Seldonian~\citep{thomas2019preventing}.  

In practice, there may be constraints that are impossible to enforce, e.g., if the DI objectives described by the user are impossible to satisfy simultaneously~\citep{kleinberg2016inherent}, or if the amount of data provided to the algorithm is insufficient to ensure fairness with high confidence.
Then, the algorithm should be allowed to return ``No Solution Found'' (\NSF) instead of a solution it does not trust. 
Let $\NSF \in \Theta$ and $g(\NSF)=0$, indicating it is always fair for the algorithm to say ``I'm unable to ensure fairness with the required confidence.'' 
%
%
Note that a fair algorithm can trivially  satisfy~\eqref{eqn: delayed impact constraint} by always returning \NSF\ instead of a model. 
%
%
Ideally, if a nontrivial fair solution exists, the algorithm should be able to find and return it given enough data. We call this property \emph{consistency} and formally define it in Section~\ref{sec: theoretical results}.

Our goal is to design a fair classification algorithm that satisfies two properties: \textbf{1)} the algorithm satisfies~\eqref{eqn: delayed impact constraint}
%
%
and \textbf{2)} the algorithm is consistent, i.e., if a nontrivial fair solution exists, the probability that the algorithm returns a fair solution (other than \NSF) converges to one as the amount of training data goes to infinity. 
%
%
In Section~\ref{sec: theoretical results}, we prove that our algorithm, \algname, satisfies both properties. 
%

\section{Methods for enforcing delayed impact}
\label{sec: delayed impact in fair classification}
According to our problem statement, a fair algorithm must ensure with high confidence that $g(\theta) \leq 0$, where $\theta$ is the returned solution and $g(\theta) = \tau - \mathbf E[I^{\pi_{\theta}}|c(X,Y,T)]$.
Because \algname only has access to historical data, $D$, only samples of $I^{\beta}$ (the delayed impact induced by predictions made by $\beta$) are available.
%
In this section, we show how one can construct i.i.d.~estimates of $I^{\pi_{\theta}}$ using samples collected using $\beta$.
Then, we show how confidence intervals can be used to derive high-confidence upper bounds on $g(\theta)$.
Lastly, we provide pseudocode for our algorithm, \algname, which satisfies~\eqref{eqn: delayed impact constraint}.

\subsection{Deriving estimates of delayed impact} 
\label{sec: deriving estimates of delayed impact} 
To begin, notice that the distribution of delayed impacts in $D$ results from using the model $\beta$ to collect data. However, we are interested in evaluating the DI of a different model, $\pi_{\theta}$.
%
%
This presents a challenging problem: given data that includes the DI when a model $\beta$ was used to make predictions, how can we estimate the DI if $\pi_{\theta}$ were used instead? 
One na\"{\i}ve solution is to simply run $\pi_{\theta}$ on the held-out data. However, this would only produce predictions $\widehat Y^{\pi_{\theta}}$, not their corresponding delayed impact.
That is, the delayed impact for each data point would still be in terms of $\beta$ instead of $\pi_{\theta}$. 

We solve this problem using \emph{off-policy evaluation} methods from the RL literature---these use data from running one \emph{policy} (decision-making model) to predict what would happen (in the long term) if a different policy had been used. 
Specifically, we use \emph{importance sampling}~\citep{Precup2001} to obtain a new random variable $\hat I^{\pi_{\theta}}$, constructed using data from $\beta$, such that $\hat I^{\pi_{\theta}}$ is an unbiased estimator of $I^{\pi_{\theta}}$:
\begin{equation}
    \label{eqn: rv equivalence}
    \mathbf{E}\left [\hat I^{\pi_{\theta}}\middle |c(X,Y,T) \right]\!=\! \mathbf{E} \left [I^{\pi_{\theta}} \middle |c(X,Y,T) \right ].
\end{equation}
For each data point, the importance sampling estimator, $\hat I^{\pi_{\theta}}$, weights the observed delayed impacts $I^{\beta}$ based on how likely the prediction $\widehat Y^{\beta}$ is under $\pi_{\theta}$. If $\pi_{\theta}$ would make the label $\widehat Y^{\beta}$ more likely, then $I^{\beta}$ is given a larger weight (at least one), and if $\pi_{\theta}$ would make $\widehat Y^{\beta}$ less likely, then $I^{\pi_{\theta}}$ is given a smaller weight (positive, but less than one). 
The precise weighting scheme is chosen to ensure that $\E [\hat I^{\pi_{\theta}}]=\E [I^{\pi_{\theta}}]$.
%
Formally, the importance sampling estimator is $\hat I^{\pi_{\theta}} =  \pi_{\theta}(X, \widehat Y^{\beta})\beta(X, \widehat Y^{\beta})^{-1} I^{\beta}$,
%
%
where the term $\pi_{\theta}(X, \widehat Y^{\beta})/\beta(X, \widehat Y^{\beta})$ is called the \emph{importance weight}. 

Next, we introduce a common and necessary assumption: the model $\pi_\theta$ can only select labels for which there is \emph{some} probability of the behavior model selecting them (Assumption~\ref{ass: support}). In Section~\ref{sec: pseudocode} we discuss a practical way in which this can be ensured. Theorem \ref{thm: is known behavior model} establishes that the importance sampling estimator is unbiased, i.e., it satisfies \eqref{eqn: rv equivalence}.
\begin{assumption}[Support]
\label{ass: support}
For all $x$ and $y$, $\pi_{\theta}(x,y) > 0$ implies that $\beta(x,y) > 0$.
\end{assumption}
%
%
\begin{theorem}\label{thm: is known behavior model}
If Assumptions~\ref{ass: markov property}--\ref{ass: support} hold, then $\mathbf{E}[\hat I^{\pi_{\theta}}|c(X,Y,T)]\!=\! \mathbf{E}[I^{\pi_{\theta}}|c(X,Y,T)]$. \emph{\textbf{Proof.}} See Appendix~\ref{app: proof of unbiased is estimate}. \qed
\end{theorem}

\subsection{Bounds on delayed impact}

This section discusses how to 
use unbiased estimates of $g(\theta)$ together with confidence intervals to derive high-confidence upper bounds on $g(\theta)$. 
While different confidence intervals for the mean can be substituted to derive these bounds, to make our method concrete, we consider the specific cases of Student's $t$-test~\citep{student1908probable} and Hoeffding's inequality~\citep{hoeffding1963probability}. 
Given a vector of $m$ i.i.d.~samples $(z_i)^m_{i=1}$ of a random variable $Z$, let $\bar Z = \frac{1}{m}\sum^m_{i=1} Z_i$ be the sample mean, $\sigma(Z_1, ..., Z_m) = \sqrt{\frac{1}{m-1}\sum^m_{i=1}(Z_i - \bar Z)^2}$ be the sample standard deviation (with Bessel's correction), and $\delta \in (0,1)$ be a confidence level. 
\begin{property}
    \label{prop: student's ttest}
    If $\sum_{i=1}^m Z_i$ is normally distributed, then
    \begin{equation}
        \Pr \left(\E[Z_i] \geq \bar Z - \frac{\sigma(Z_1, ..., Z_m)}{\sqrt{m}}t_{1-\delta, m-1}\right) \geq 1-\delta,
    \end{equation}
    where $t_{1-\delta, m-1}$ is the $1-\delta$ quantile of the Student's $t$ distribution with $m - 1$ degrees of freedom.
    \emph{\textbf{Proof.}} See the work of~\citet{student1908probable}. \qed
\end{property}
Property~\ref{prop: student's ttest} can be used to obtain a high-confidence upper bound for the mean of $Z$:
$U_{\texttt{ttest}} (Z_1, ..., Z_m) \coloneqq \bar Z + \frac{\sigma(Z_1, ..., Z_m)}{\sqrt{m}}t_{1-\delta, m-1}$.
Let $\hat g$ be a vector of  i.i.d.~and unbiased estimates of $g(\theta)$. Once these are computed (using importance sampling as described in Section~\ref{sec: deriving estimates of delayed impact}), they can be provided to $U_{\texttt{ttest}}$ to derive a high-confidence upper bound on $g(\theta)$: 
$ 
    \Pr(\tau - \mathbf E[\hat I^{\pi_{\theta}} | c(X, Y, T)] \leq U_\texttt{ttest}(\hat g)) \geq 1-\delta.
$ 
Importantly, inequalities based on Student's $t$-test only hold exactly if the distribution of $\sum Z_i$ is normal.\footnote{The central limit theorem states that the distribution of the sample mean converges to a normal distribution as $m \rightarrow \infty$ regardless of the distribution of $Z_i$, making this approximation reasonable for sufficiently large $m$.}
Our strategy for deriving high-confidence upper bounds for definitions of DI is general and can be applied using other confidence intervals (for the mean) as well. To illustrate this, in Appendix~\ref{app: bound di using hoeff} we describe an alternative bound
based on Hoeffding's inequality~\citep{hoeffding1963probability}, which replaces the normality assumption with the weaker assumption that $\hat g$ is bounded. This results in a different function for the upper bound, $U_\texttt{Hoeff}$.

\subsection{Complete algorithm}
\label{sec: pseudocode}
\begin{algorithm}[tb]
\caption{\algname}
\label{alg: main algorithm}
\textbf{Input}: \textbf{1)} $D = \{(X_i, Y_i, T_i, \widehat Y^\beta_i, I^\beta_i)\}^n_{i=1}$;
%
%
\textbf{2)} confidence level $\delta$; 
\textbf{3)} tolerance value $\tau$;
\textbf{4)} behavior model $\beta$; 
 and
\textbf{5)} $\texttt{Bound} \in\{ \texttt{Hoeff}, \texttt{ttest}\}$. 
\\
\textbf{Output}: Solution $\theta_c$ or \NSF.
\begin{algorithmic}[1] 
\STATE $D_c, D_f \leftarrow \texttt{partition}(D)$\label{line: partition dataset}
\STATE $n_{D_f} = \texttt{length}(D_f); \quad \hat g \leftarrow \langle \ \rangle$
\STATE $\theta_c \leftarrow \arg\min_{\theta\in\Theta}  \texttt{cost}(\theta,D_c, \delta, \tau, \beta, \texttt{Bound}, n_{D_f})$\label{line: candidate selection} 
%
%
%
\FOR{$j\in\{1,...,n_{D_f}\}$}
\STATE Let $(X_j, Y_j, T_j, \widehat Y^\beta_j, I^\beta_j)$ be the $j^\text{th}$ data point in $D_f$
\STATE \textbf{if} $ c(X_j, Y_j, T_j)$ is \texttt{True} \textbf{then} $\hat g$.append$\left(\tau - \frac{\pi_{\theta_c}(X_j,\widehat
Y^\beta_j)}{\beta(X_j, \widehat Y^\beta_j)}I^\beta_j\right)$ \textbf{end if }
\ENDFOR 
%
%
\STATE \textbf{if} \texttt{Bound} is \texttt{Hoeff} \textbf{then} $U = U_\texttt{Hoeff} (\hat g)$ \textbf{ else if} \texttt{Bound} is \texttt{ttest} \textbf{then} $U = U_{\texttt{ttest}}(\hat g)$ \textbf{end if}
%
\STATE \textbf{if} $U\geq 0$ \textbf{then return} \NSF\ \textbf{else return} $\theta_c$ 
\end{algorithmic}
\end{algorithm}
Algorithm~\ref{alg: main algorithm} provides pseudocode for \algname. 
%
%
Our algorithm has three main steps: First, the dataset $D$ is divided into two datasets (line~1).  
In the second step, \emph{candidate selection} (line~3), the first dataset, $D_c$, is used to find and train a model, called the \emph{candidate solution}, $\theta_c$. This step is detailed in Algorithm~\ref{alg: cost function} (see Appendix \ref{app: full algorithm}).
In the \emph{fairness test} (lines~4--9), the dataset $D_f$ is used to compute unbiased estimates of $g(\theta_c)$ using the importance sampling method described in Section~\ref{sec: deriving estimates of delayed impact}. These estimates are used to calculate a $(1{-}\delta)$-confidence upper bound, $U$, on $g(\theta_c)$, using Hoeffding's inequality or Student's $t$-test (line~8). 
Finally, $U$ is used to determine whether $\theta_c$ or \NSF\ is returned.

Recall that for importance sampling to produce unbiased estimates of $g(\theta_c)$, Assumption~\ref{ass: support} must hold.
%
%
%
To ensure this, 
we restrict $\Theta$ to only include solutions that satisfy Assumption~\ref{ass: support}:
%
\begin{assumption}
\label{ass: restrict feasible set}
Every $\theta \in \Theta$ satisfies Assumption~\ref{ass: support}. 
\end{assumption}
%
%
%
One way to achieve this is by ensuring that the behavior model $\beta$ has full support; that is, for all $x$ and $\hat y$, $\beta(x,\hat y) > 0$. Notice that many supervised learning algorithms already place non-zero probability on every label, e.g., when using Softmax layers in neural networks. In these commonly-occurring cases, Assumption~\ref{ass: restrict feasible set} is trivially satisfied.

%
%
%

%

\section{Theoretical results}
\label{sec: theoretical results}

This section shows that \textbf{1)}~\algname ensures delayed impact fairness with high confidence; that is, it is guaranteed to satisfy DI constraints as defined in~\eqref{eqn: delayed impact constraint}, and \textbf{2)}~given reasonable assumptions about the DI objectives, \algname is consistent.
%
%
To begin, we make an assumption related to the confidence intervals used to bound $g(\theta_c)$, where $\theta_c$ is returned by candidate selection.
Specifically, we assume that the requirements related to Student's $t$-test (Property~\ref{prop: student's ttest}) or Hoeffding's inequality (Property~\ref{prop: hoeffding}; see Appendix~\ref{app: bound di using hoeff}) are satisfied. 
%
Let $\mathrm{Avg}(Z){=}\frac{1}{n_Z}\!\sum_{i=1}^{n_Z}\!Z_i$ be the average of a size $n_Z$ vector $Z$. 
\begin{assumption}
    \label{ass: estimates match bound}
    If \texttt{Bound} is \texttt{Hoeff}, then for all $j \in\{1, ..., k\}$, each estimate in $\hat g_j$ (in Algorithm~\ref{alg: main algorithm multiple constraints}) is bounded in some interval $[a_j, b_j]$. If \texttt{Bound} is \texttt{ttest}, then each Avg($\hat g_j$) is normally distributed. 
\end{assumption}

\begin{theorem}
    \label{thm: fairness guarantee}
    Let $(g_j{})^k_{j=0}$ be a sequence of DI constraints, where $g_j: \Theta \rightarrow \mathbb R$, and let $(\delta_j)^k_{j=1}$ be a corresponding sequence of confidence levels, where each $\delta_j \in (0,1)$. If Assumptions~\ref{ass: markov property}, \ref{ass: switch prediction}, \ref{ass: restrict feasible set}, and \ref{ass: estimates match bound} hold, and if algorithm $a$ corresponds to Algorithm~\ref{alg: main algorithm multiple constraints}, then for all $j \in \{1, ..., k\}$,  $\Pr(g_j(a(D)) \leq 0) \geq 1-\delta_j$. \emph{\textbf{Proof.}} See Appendix~\ref{app: fairness guarantee}. \qed
\end{theorem}
%
\algname satisfies Theorem~\ref{thm: fairness guarantee} if the solutions it produces satisfy~\eqref{eqn: delayed impact constraint}, i.e., if $\forall j \in \{1, ..., k\}, \ \Pr(g_j(a(D)) \leq 0) \geq 1-\delta_j $, where $a$ is Algorithm~\ref{alg: main algorithm multiple constraints}.
Because Algorithm~\ref{alg: main algorithm multiple constraints} is an extension of Algorithm~\ref{alg: main algorithm} to multiple constraints, it suffices to show that Theorem~\ref{thm: fairness guarantee} holds for Algorithm~\ref{alg: main algorithm multiple constraints}. Next, we show that \algname is consistent, i.e., that when a fair solution exists, the probability that the algorithm returns a solution other than \NSF\, converges to $1$ as the amount of training data goes to infinity:
\begin{theorem}[Consistency guarantee]
    \label{thm: consistency}
    If Assumptions~\ref{ass: markov property}--\ref{ass: consistent loss estimator} (\ref{ass: piecewise Lipschitz}--\ref{ass: consistent loss estimator} are given in Appendix~\ref{app: consistency proof}) hold, then $\lim_{n\rightarrow \infty} \Pr(a(D) \neq \NSF, g(a(D)) \leq 0) = 1$. \emph{\textbf{Proof}}. \citet{metevier2019offline} provide a similar proof for a Seldonian contextual bandit algorithm. 
    %
    Appendix~\ref{app: consistency proof} adapts their proof.
\end{theorem}
%
%

To prove Theorem~\ref{thm: consistency}, we make a few mild assumptions: 
\textbf{1)} a weak smoothness assumption about the output of Algorithm~\ref{alg: cost function}; 
\textbf{2)} that at least one fair solution exists that is not on the fair-unfair boundary; and \textbf{3)} that the sample loss in Algorithm~\ref{alg: cost function} converges almost surely to the true loss. 
In Appendix \ref{app: assumption intuitions}, we provide a high-level discussion on the meaning and implications of Assumptions~1--8. Our goal is to show that these assumptions are standard in the ML literature and reasonable in many real-life settings. 

\section{Empirical evaluation}
\label{sec: experiments}
We empirically investigate three research questions: \textbf{RQ1:}~Does \algname enforce DI constraints with high probability, while existing fairness-aware algorithms do not? \textbf{RQ2:}~What is the cost of enforcing DI constraints? \textbf{RQ3:}~How does \algname perform when predictions have little influence on DI relative to other factors outside of the model's control? 


\begin{figure*}[b!!!]
\centering
\includegraphics[width=\textwidth]{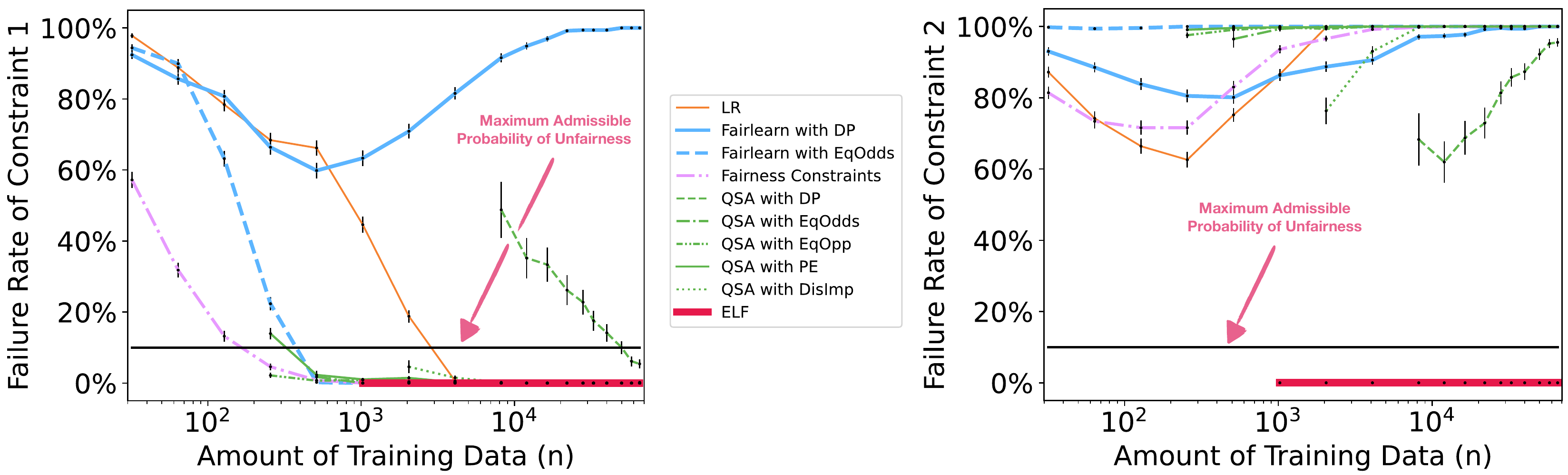}
\caption{Algorithms' failure rates with respect to the DI constraints associated with White people (on the left) and Black people (on the right), as a function of $n$. The black horizontal lines indicate the maximum admissible probability of unfairness, $\delta_0=\delta_1=10\%$.}
\label{fig:comparison-failureRate}
\end{figure*}

We consider a classifier tasked with making predictions about people in the United States foster care system; for example, whether youth currently in foster care are likely to get a job in the near future. These predictions may have a delayed impact on the person's life if, for instance, they 
influence whether that person receives additional financial aid. 
Here, the goal is to ensure that a trained classifier is fair with respect to DI when considering race.
Our experiments use two data sources from the National Data Archive on Child Abuse and Neglect~\citep{NDACAN2021} that do not include personally identifiable information or offensive content:
\textbf{1)}~the Adoption and Foster Care Analysis and Reporting System---a
dataset containing demographic and foster care-related information about youth; and \textbf{2)}~the National Youth in Transition Database (Services and Outcomes)---a dataset containing information about the financial and educational status and well-being of youth over time and during their transition from foster care to independent adulthood.

In this setting, the feature vector $X$ contains five attributes related to the job and educational status of a person in foster care. 
The sensitive attribute, $T$, corresponds to race---whether a person identifies as White or Black. The classifier is tasked with predicting a binary label, $Y$, denoting whether a person has a full-time job after leaving the program. 
We modify this dataset by resampling the original labels in a way that ensures the likelihood of being employed after leaving the program depends on \emph{both} $X$ and $T$. 
The behavior model, $\beta$, corresponds to a logistic regression classifier.

To investigate our research questions, we evaluate \algname in settings where predictions made by a classifier may have different levels of influence on DI. 
We model such settings by constructing a parameterized and synthetic definition of DI.
Let $I_i^\psi$ be the delayed impact for 
person $i$ 
if a classifier, $\psi$, outputs the prediction $\widehat Y^\psi_i$ given $X_i$. Here, $\psi(x, \hat y) \coloneqq \Pr(\widehat Y_i^\psi{=}\hat y|X_i{=}x)$. We define $I_i^\psi$ as
\begin{equation}
\label{eqn: delayed impact experiments}
    I_i^\psi =
    \begin{cases} 
    \alpha\widehat Y_i^\psi + (1-\alpha)\mathcal{N}(2,0.5)\,\,\,\,\,\,\text{if}\,\,\,T_i = 0\\
    \alpha\widehat Y_i^\psi + (1-\alpha)\mathcal{N}(1,1)\,\,\,\,\,\,\,\,\,\,\,\text{if}\,\,\,T_i = 1,
    \end{cases}
\end{equation}
where $\alpha$ regulates whether DI is strongly affected by a classifier's predictions or if predictions have little to no delayed impact.
We refer to the former setting as one with \emph{high prediction-DI dependency} and to the latter setting as one with \emph{low prediction-DI dependency}. 
As $\alpha$ goes to zero, predictions made by the classifier have no DI. 
In our experiments, we vary $\alpha$ from zero to one in increments of $0.1$; for each value of $\alpha$, we construct a corresponding dataset by using~\eqref{eqn: delayed impact experiments} to assign a measure of DI to each instance in the foster care dataset. When doing so, $\psi$ is defined as the behavior model, $\beta$.

We wish to guarantee with high probability that the DI caused by a new classifier, $\pitheta$, is better than the DI resulting from the currently deployed classifier, $\beta$. This guarantee should hold simultaneously for both races: White (instances where $T=0$) and Black (instances where $T=1$).
We model this requirement via two DI objectives, $g_0$ and $g_1$. Let $t \in \{0,1\}$ and  $g_t(\theta) \coloneqq \tau_t - \mathbf E[I^\pitheta | T = t]$, where $\tau_t = \frac{1}{n_t}\sum_{d=1}^n I^\beta_d\llbracket T_d = t \rrbracket$ is the average DI caused by $\beta$ on people of race $T=t$ and where $n_t=\sum_{d=1}^n\llbracket T_d = t \rrbracket$.
%
The confidence levels $\delta_0$ and $\delta_1$ associated with these objectives are set to $0.1$. 


\textbf{RQ1: Preventing delayed-impact unfairness.}
%
To study RQ1, we evaluate whether \algname can prevent DI unfairness with high probability, and whether existing algorithms fail.
\algname is the \emph{first} method capable of ensuring delayed-impact fairness in the model-free setting, i.e., when an analytical model describing the relationship between predictions and DI is not known and the algorithm only has access to historical data. 
To the best of our knowledge, no other method in the literature ensures DI fairness in this setting (see Section~\ref{sec: related work}). Existing methods that ensure static fairness definitions, however, can be applied in this setting---i.e., when only historical data is available. Thus, we compare \algname with three state-of-the-art fairness-aware algorithms that are designed to enforce static constraints: \textbf{1)}~Fairlearn~\citep{agarwal2018reductions}, \textbf{2)}~Fairness Constraints~\citep{zafar2017fairness}, and \textbf{3)}~quasi-Seldonian algorithms (QSA)~\citep{thomas2019preventing}. We consider five static fairness constraints: demographic parity~(DP), equalized odds~(EqOdds), disparate impact~(DisImp), equal opportunity~(EqOpp), and predictive equality~(PE)~\citep{chouldechova2017fair,dwork2012fairness,hardt2016equality}.\footnote{We task each competing method with enforcing the constraints analyzed in those methods' original papers.}
We also compare \algname with a fairness-unaware algorithm: logistic regression~(LR).


\begin{figure*}[t!!!]
\centering
\begin{subfigure}{0.4\columnwidth}
  \centering
  \includegraphics[width=\columnwidth]{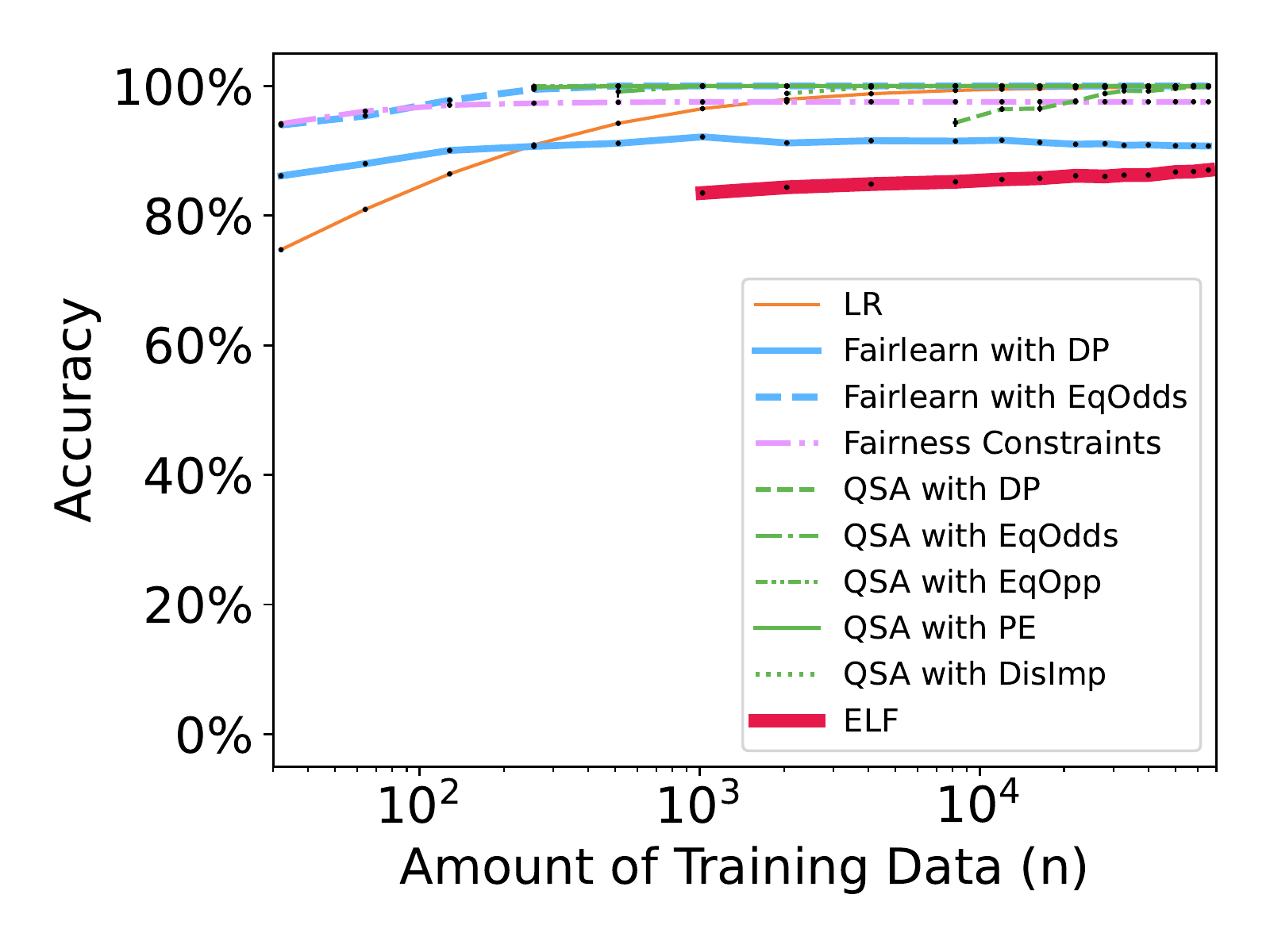}
  \caption{Accuracy.}
  \label{fig:comparison-accuracy}
\end{subfigure}
\begin{subfigure}{0.4\columnwidth}
  \centering
  \includegraphics[width=\columnwidth]{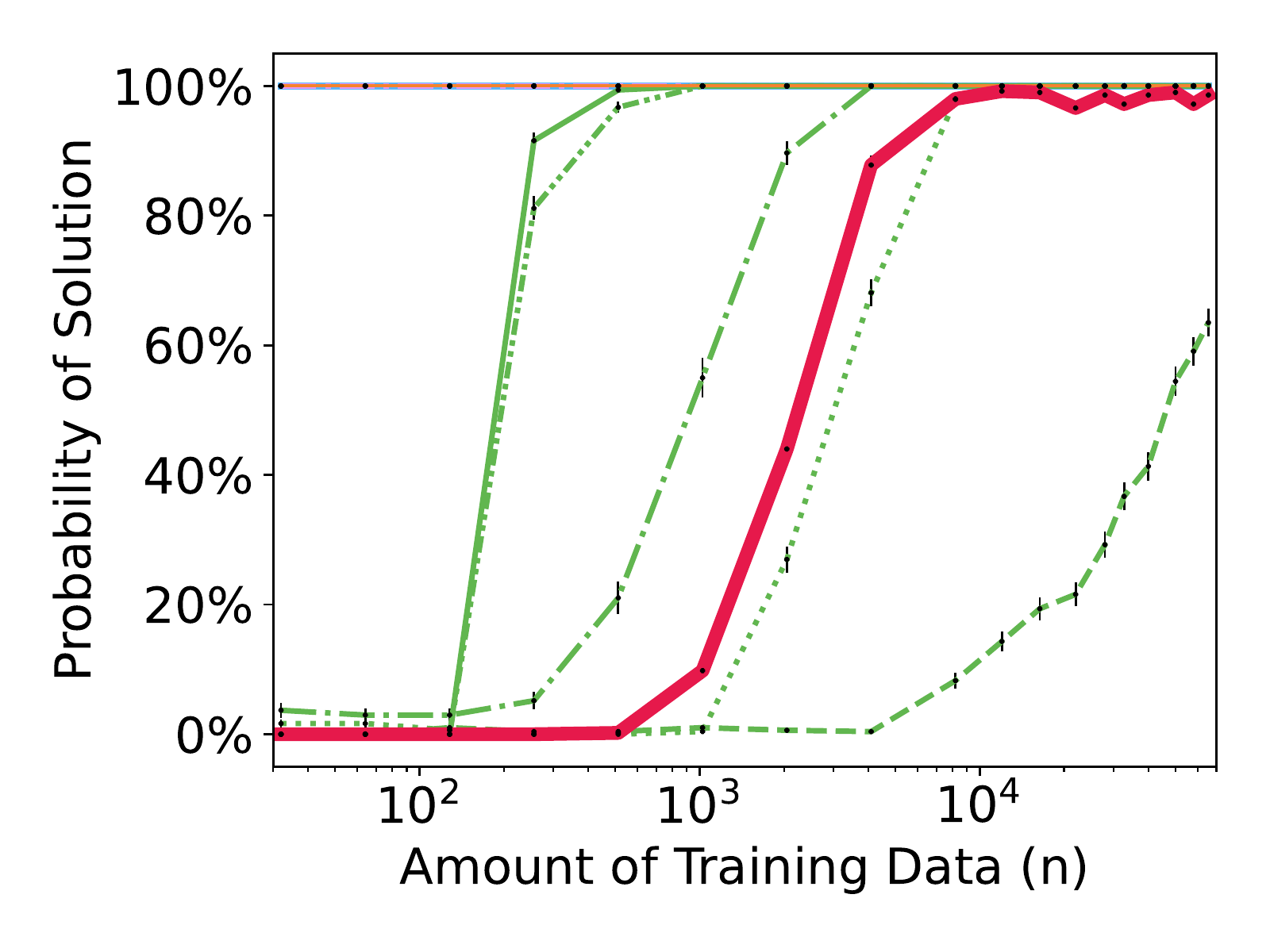}
  \caption{Probability of returning solution.}
  \label{fig:comparison-probSolution}
\end{subfigure}
\caption{On the left, the accuracy of solutions returned by algorithms (subject to different fairness constraints) as a function of $n$. On the right, the probability that these algorithms return a solution.}
\label{fig:comparison-PSandACC}
\end{figure*}

\looseness-1
This section studies how often each fairness-aware algorithm 
returns an unfair model (with respect to the DI constraints) as a function of the amount of training data, $n$.
We refer to the probability that an algorithm returns an unfair model as its \emph{failure rate}. To measure the failure rate, we compute how often the classifiers returned by each algorithm are unfair when evaluated on a significantly larger dataset, to which the algorithms do not have access during training time. To investigate how failure rates are influenced by the level of prediction-DI dependency, we vary $\alpha$ between 0 and 1.
Due to space constraints, below we discuss one representative experiment conducted by setting $\alpha=0.9$. The qualitative behavior of all algorithms for other values of $\alpha$ is similar. The complete set of results and details about our \algname implementation can be found in Appendix~\ref{app: experiments}.

Figure~\ref{fig:comparison-failureRate} presents the failure rate of each algorithm as a function of the amount of training data. We computed all failure rates and corresponding standard errors over 500 trials.
Notice that the solutions returned by \algname are \emph{always fair} with respect to the DI constraints.\footnote{\algname does not return solutions
with 
$n<1,\!000$ 
because it cannot ensure DI fairness with high confidence.}
This is consistent with \algname's theoretical guarantees, which ensure with high probability that the solutions it returns satisfy all fairness constraints. 
Existing methods that enforce static fairness criteria, by contrast, either \textbf{1)}~\emph{always fail} to satisfy both DI constraints;
or \textbf{2)}~\emph{always fail} to satisfy one of the DI constraints---the one related to delayed impact on Black youth---while often failing to satisfy the other constraint. 

\underline{RQ1}: \emph{Our experiment supports the hypothesis that with high probability \algname is fair with respect to DI objectives, and that comparable fairness-aware techniques do not ensure delayed impact fairness.}


\paragraph{RQ2: The cost of ensuring delayed-impact fairness.}

Previously, we 
showed that \algname is capable of satisfying DI constraints with high probability. But depending on the data domain, this may come at a cost.
First, there may be a trade-off between satisfying fairness constraints and optimizing accuracy. 
In Appendix~\ref{app: extensions} we show how \algname can be tasked with satisfying DI fairness constraints while also bounding accuracy loss.
%
Here, we investigate the impact that enforcing DI constraints has on accuracy.
Figure~\ref{fig:comparison-accuracy} presents the accuracy of classifiers returned by different algorithms as a function of $n$.\footnote{As before, we use $\alpha{=}0.9$ and 500 trials. Results for other values of $\alpha$ are in Appendix~\ref{app: experiments}.}
%
In these experiments, we bound accuracy loss via an additional constraint requiring that \algname's solutions have accuracy of at least 75\%.
Under low-data regimes ($n{=}1,\!000$), \algname's accuracy is 83\%, while competing methods (with no DI fairness guarantees) have accuracy higher than 90\%.  
%
Importantly, however, notice that whenever competing methods have higher accuracy than ours, they \emph{consistently return unfair solutions}. \algname, by contrast, ensures that \emph{all fairness constraints are satisfied} with high probability and always returns solutions with accuracy above the specified threshold (see Figure~\ref{fig:comparison-failureRate}). 
Furthermore, notice that as $n$ 
increases, \algname's accuracy approaches that of the other techniques.

Second, there may be a trade-off between the amount of training data and the confidence that a fair solution has been identified. 
Recall that some algorithms (including ours) may not return a solution if they cannot ensure fairness with high confidence. 
%
%
Here we study how often each algorithm identifies and returns a candidate solution as a function of $n$. 
Figure~\ref{fig:comparison-probSolution} shows that \algname returns solutions with $91\%$ probability when given just $n=4,\!096$ training instances. 
As 
$n$ increases, the probability of \algname returning solutions increases rapidly. Although three competing techniques (Fairlearn, Fairness Constraints, and LR) always return solutions, independently of the amount of training data, these solutions \emph{never satisfy both DI constraints} (see Figure~\ref{fig:comparison-failureRate}).
QSA enforcing static fairness constraints often returns candidate solutions with less training data than \algname; these solutions, however, also \emph{fail to satisfy both DI constraints} simultaneously.

\underline{RQ2}: \emph{While there is a cost to enforcing DI constraints depending on the data domain, \algname succeeds in its primary objective: to ensure DI fairness with high probability, without requiring unreasonable amounts of data, while also bounding accuracy loss.}

\paragraph{RQ3: Varying prediction-DI dependency.}

\begin{figure*}[t!!!]
\centering
\begin{subfigure}{0.32\columnwidth}
  \centering
  \includegraphics[width=\columnwidth]{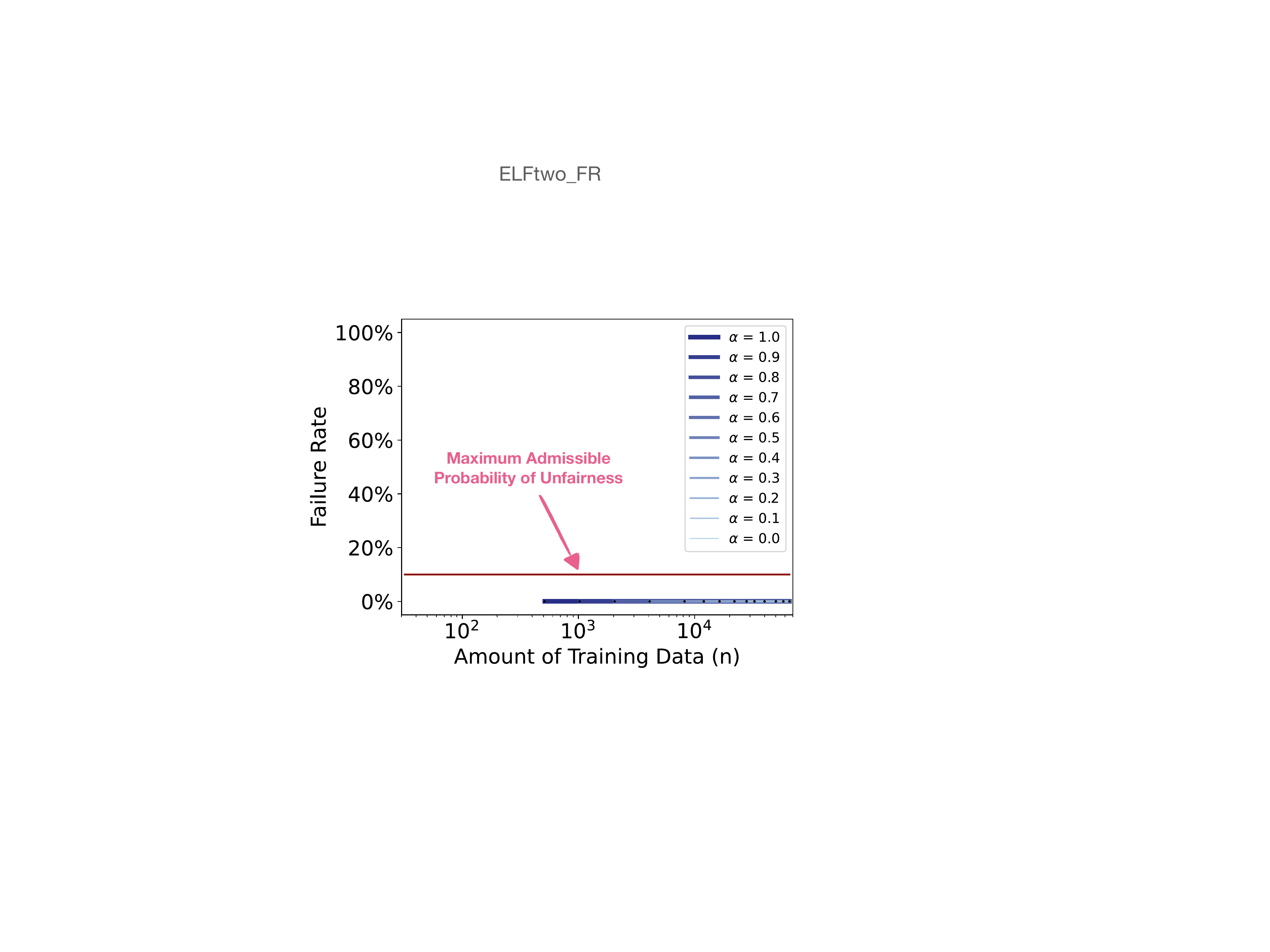}
  \caption{Failure rate.}
  \label{fig:ELFtwo-failureRate}
\end{subfigure}
\begin{subfigure}{0.32\columnwidth}
  \centering
  \includegraphics[width=\columnwidth]{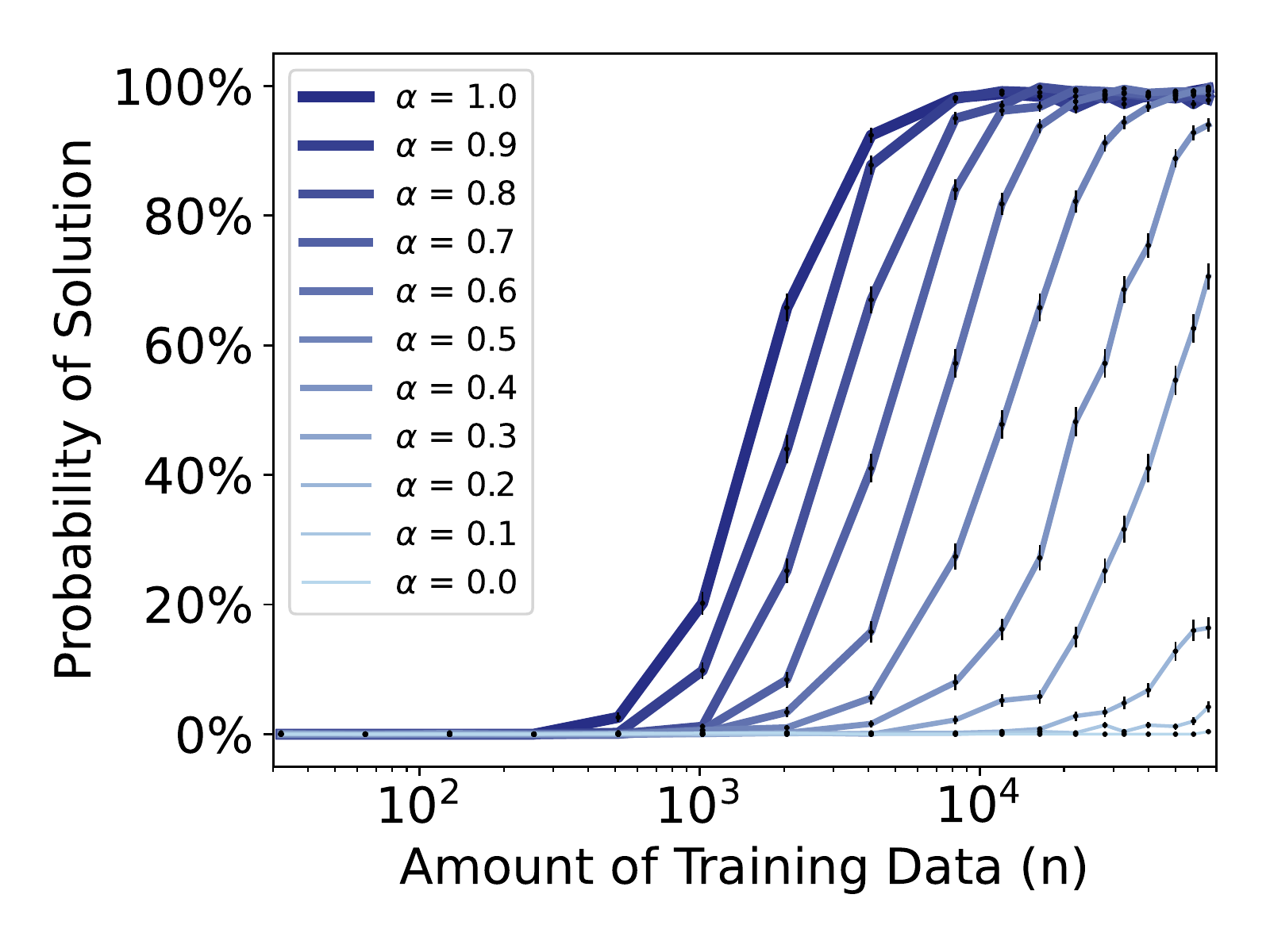}
  \caption{Probability of returning solution.}
  \label{fig:ELFtwo-probSolution}
\end{subfigure}
\begin{subfigure}{0.32\columnwidth}
  \centering
  \includegraphics[width=\columnwidth]{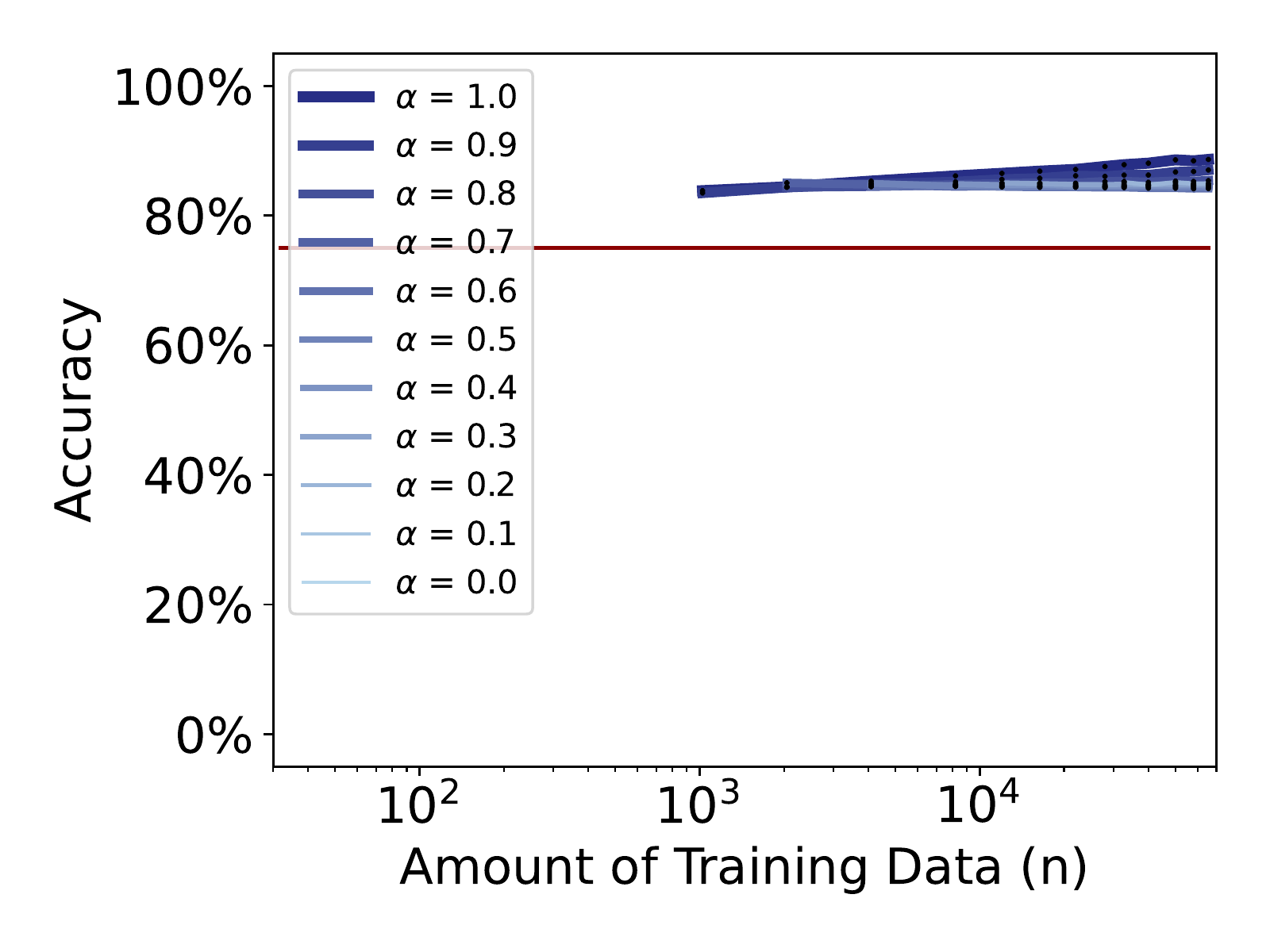}
  \caption{Accuracy.}
  \label{fig:ELFtwo-accuracy}
\end{subfigure}
\caption{\algname's performance in settings with different levels of prediction-DI dependency, as a function of $n$.}
\label{fig:ELFtwo}
\end{figure*}


Finally, we investigate \algname's performance (in terms of failure rate, probability of returning solutions, and accuracy) in settings with varied levels of prediction-DI dependency. These include challenging cases where predictions have little influence on DI relative to other factors outside of the model's control.

We first study \algname's failure rate (probability that the returned model is unfair) for different values of $\alpha$. 
Figure~\ref{fig:ELFtwo-failureRate} shows that \algname \emph{never} returns unfair models, independent of $\alpha$, confirming empirically that \algname's high-probability fairness guarantees hold in settings with a wide range of qualitatively different DI characteristics. 
Next, we investigate how often \algname identifies and returns a solution for various values of $\alpha$.
We expect that if the predictions made by a classifier have little to no DI (i.e., for low values of $\alpha$), it becomes harder---if not impossible---for \algname to be confident that it has identified a model satisfying all DI constraints. This is because it becomes harder to separate the small delayed impact of predictions from noise. 
%
In these cases, we expect it would be less likely for \algname to return solutions.
Figure~\ref{fig:ELFtwo-probSolution} illustrates this behavior. Notice that \algname returns solutions under all $\alpha>0$ if given sufficient training data. However, as expected, the probability that it returns a solution decreases as $\alpha$ approaches zero.
Lastly, we investigate how the amount of prediction-DI dependency affects the accuracy of \algname's solutions.
Figure~\ref{fig:ELFtwo-accuracy} shows the model accuracy, for various values of $\alpha$, as a function of $n$. The accuracy trade-off is more evident when \algname must satisfy challenging DI objectives, as $\alpha$ approaches zero. In such cases, accuracy decreases from 90\% to 84\%. Importantly, however, even though a trade-off exists, notice that our method is successful at bounding the accuracy of returned solutions while ensuring DI fairness constraints with high confidence.

\underline{RQ3}: \emph{Our experiments confirm that \algname performs well in a wide range of settings, with various levels of prediction-DI dependency. 
Even though ensuring fairness may impact accuracy and the probability of finding solutions, these unavoidable, domain-specific trade-offs do not affect \algname's fairness guarantees. In our experiments, all returned models satisfy both DI constraints.}


\section{Related work}
\label{sec: related work}
Most prior work on the social implications of ML study static fairness without considering the long-term impact of model decisions \citep{calders2009building, zafar2017fairness, hardt2016equality, dwork2012fairness}.
However, there exists a growing body of work that examines the long-term impact of fairness in ML \cite{d2020fairness, hu2018short, hu2018welfare, liu2018delayed, heidari2019long, zhang2020long, mouzannar2019fair}. 
 In this paper, we build upon this prior work and present the first method that uses historical data to train a classifier with high-confidence delayed-impact fairness guarantees when the analytic model of the relationship between model predictions and delayed impact is not known \emph{a priori}.

\citet{wen2021algorithms} and~\citet{tang2021bandit} present work similar to ours.
%
\citet{wen2021algorithms} propose modeling delayed impact using a Markov decision process (MDP) with two different reward functions: one for the decision-maker, e.g., a bank, and another for each individual, e.g., loan applicant. \citet{wen2021algorithms} introduce algorithms that are able to estimate near-optimal policies (in terms of cumulative reward of the decision-maker) while enforcing static fairness constraints (e.g., demographic parity and equal opportunity). Importantly, \citet{wen2021algorithms} introduce a method that ensures that \emph{static} fairness constraints hold for all time during a sequence of decisions. We, by contrast, study the orthogonal problem of ensuring fairness with respect to user-defined \emph{delayed-impact} measures. 
%
%
The method proposed by \citet{tang2021bandit}, unlike \algname, considers the online multi-armed bandit learning setting in which there are no features---the algorithm does not differentiate individuals within a group while making decisions. We, by contrast, tackle the problem of high-confidence fairness in the classification setting.

Work by~\citet{d2020fairness, ge2021towards}, and~\citet{hu2022achieving} study similar but orthogonal problem settings. 
%
%
\citet{ge2021towards} and \citet{hu2022achieving} train classifiers that satisfy static fairness constraints in non-stationary settings (e.g., a recommendation system where a person's interests may change over time). Importantly, both \citet{ge2021towards} and~\citet{hu2022achieving} require prior knowledge of analytic, accurate models of the environment---for example, in the form of probabilistic graphical models. Our method, by contrast, does not require a model or simulator of the environment, nor prior knowledge about the relationship between a classifier's predictions and the resulting delayed impact. 
%
%
\citet{d2020fairness}'s goal is not to propose a new method; instead, they \emph{evaluate} the DI resulting from a given classifier's predictions, under the assumption that an accurate simulator of the environment is available. By contrast, we propose a new method for \emph{training} classifiers that ensure that DI fairness constraints are satisfied, without requiring accurate simulators. 

In another line of work, researchers have shown that the fairness of ML algorithms can be improved by manipulating the training data, rather than the learning algorithm. 
This goal can be achieved, for example, by removing data that violates fairness properties~\citep{Verma21} or by inserting data inferred using fairness properties~\citep{Salimi19}. 
Again, while these methods can improve the fairness of learned models, they were designed to enforce static fairness constraints and do not enforce fairness with respect to the delayed impact resulting from deploying such learned models. 

Lastly, this paper introduces a method that extends the existing body of work on \textit{Seldonian algorithms~\citep{thomas2019preventing}}.
Seldonian algorithms provide fairness guarantees with high probability and have been shown to perform well in real-world applications given reasonable amounts of training data \citep{thomas2019preventing, metevier2019offline}.
They also---by construction---provide a straightforward way for users to define multiple notions of fairness that can be simultaneously enforced \citep{thomas2019preventing}. 
%
%
The technique introduced in this paper (\algname) is the first supervised-learning Seldonian algorithm capable of providing high-probability fairness guarantees in terms of delayed impact.

\subsection*{Acknowledgments}
Research reported in this paper was sponsored in part by gifts and grants from Adobe, Meta Research, Google, the National Science Foundation award no. 2018372,  the U.S. National Science Foundation under grant no. CCF-1763423, and the DEVCOM Army Research Laboratory under Cooperative Agreement W911NF-17-2-0196 (ARL IoBT CRA). The views and conclusions contained in this document are those of the authors and should not be interpreted as representing the official policies, either expressed or implied, of the Army Research Laboratory or the U.S. Government. The U.S. Government is authorized to reproduce and distribute reprints for Government purposes notwithstanding any copyright notation herein.

\bibliographystyle{plainnat}
\bibliography{references}

\newpage 
\appendix

\section{Proof of Theorem~\ref{thm: is known behavior model}}
\label{app: proof of unbiased is estimate}
\begin{proof}
At a high level, we start with $\mathbf E[\hat I^{\pi_{\theta}}|c(X,Y,T)]$ and, through a series of transformations involving substitution, laws of probability, and Assumptions~\ref{ass: markov property}--\ref{ass: support}, obtain $\mathbf E[I^{\pi_{\theta}}|c(X,Y,T)]$. To simplify notation, throughout this proof we let $C=c(X,Y,T)$. 
Also, for any random variable $Z$, let $\operatorname{supp}(Z)$ denote the support of $Z$ (e.g., if $Z$ is discrete, then $\operatorname{supp}(Z) = \{z: \Pr(Z=z) > 0\}$).
To begin, we substitute the definition of $\hat I^{\pi_{\theta}}$ and expand this expression using the definition of expected value:
\begin{align}
    \mathbf{E}[\hat I^{\pi_{\theta}}|C] =& \mathbf{E}\left [\frac{\pi_{\theta}(X, \widehat Y^{\beta})}{\beta(X, \widehat Y^{\beta})}I^{\beta}\middle |C\right ]\\
    %
    =& \smashoperator{\sum_{(x,y,t,\hat y, i) \in \operatorname{supp}(X,Y,T,\widehat Y^\beta, I^\beta)}} \Pr(X{=}x,Y{=}y,T{=}t, \widehat Y^{\beta}{=}\hat y, I^{\beta}{=}i|C) \frac{\pi_{\theta}(x, \hat y)}{\beta(x, \hat y)}i.\label{eqn: is thm exp of Ihat pi c}
    %
\end{align} 
Using the chain rule repeatedly, we can rewrite the joint probability in~\eqref{eqn: is thm exp of Ihat pi c} as follows:
\begin{align}
    &\Pr(X{=}x,Y{=}y,T{=}t, \widehat Y^{\beta}{=}\hat y,  I^{\beta}{=}i|C)\\
    =&\Pr(I^{\beta}{=}i|X{=}x,Y{=}y,T{=}t, \widehat Y^{\beta}{=}\hat y,C) \Pr(X{=}x,Y{=}y,T{=}t, \widehat Y^{\beta}{=}\hat y | C)\\
    =& \Pr(I^{\beta}{=}i|X{=}x,Y{=}y,T{=}t, \widehat Y^{\beta}{=}\hat y,C)\Pr(\widehat Y^{\beta}{=}\hat y | X{=}x,Y{=}y,T{=}t, C)\Pr(X{=}x,Y{=}y,T{=}t| C).
\end{align}
%
%
Under Assumption~\ref{ass: markov property}, $\Pr(\widehat Y^{\beta}{=}\hat y|X{=}x,T{=}t,Y{=}y,C) = \Pr(\widehat Y^{\beta}{=}\hat y|X{=}x)$, which is the definition of $\beta(x,\hat y)$. 
We perform this substitution and simplify by cancelling out the $\beta$ terms: 
\begin{align}
    \mathbf{E}[\hat I^{\pi_{\theta}}|C] =&
     \smashoperator{\sum_{(x,y,t,\hat y,i) \in \operatorname{supp}(X, Y, T, \widehat Y^\beta,I^\beta)}} \Pr \left (I^{\beta}{=}i|X{=}x,Y{=}y,T{=}t, \widehat Y^{\beta}{=}\hat y,C \right )\beta(x,\hat y)\Pr \left (X{=}x,Y{=}y,T{=}t|C\right)   \frac{\pi_{\theta}(x, \hat y)}{\beta(x, \hat y)}i\\%
    %
    %
    =& \smashoperator{\sum_{(x,y,t,\hat y,i) \in \operatorname{supp}(X, Y, T, \widehat Y^\beta, I^\beta)}}  \Pr\left (I^{\beta}{=}i|X{=}x,Y{=}y,T{=}t, \widehat Y^{\beta}{=}\hat y,C\right )\Pr\left(X{=}x,Y{=}y,T{=}t|C\right) \pi_{\theta}(x,\hat y)i.\label{eqn: is thm simplification}
\end{align}
Note that under Assumption~\ref{ass: markov property}, $\pi_{\theta}(x,\hat y)$ can be rewritten as $\Pr(\widehat Y^{\pi_{\theta}} {=} \hat y|X{=}x, Y{=}y, T{=}t, C)$. Using the multiplication rule of probability, we can combine this term with the  $\Pr(X{=}x,Y{=}y,T{=}t|C)$ term in~\eqref{eqn: is thm simplification} to obtain the joint probability $\Pr(X{=}x,Y{=}y,T{=}t,\widehat Y^{\pi_{\theta}}{=}\hat y|C)$. 
By Assumption~\ref{ass: switch prediction}, we can substitute $\Pr(I^{\beta}{=}i|X{=}x,Y{=}y,T{=}t, \widehat Y^{\beta}{=}\hat y,C)$ for $\Pr(I^{\pi_{\theta}}{=}i|X{=}x,Y{=}y,T{=}t, \widehat Y^{\pi_{\theta}}{=}\hat y,C)$. We substitute these terms into~\eqref{eqn: is thm simplification} and apply the multiplication rule of probability once more:
\begin{align}
    \mathbf{E}[\hat I^{\pi_{\theta}}|C] =& \smashoperator{\sum_{(x,y,t,\hat y,i) \in \operatorname{supp}(X, Y, T, \widehat Y^\beta, I^\beta)}}  \Pr(I^{\pi_{\theta}}{=}i|X{=}x,Y{=}y,T{=}t, \widehat Y^{\pi_{\theta}}{=}\hat y,C)\Pr(X{=}x,Y{=}y,T{=}t,\widehat Y^{\pi_{\theta}}=\hat y|C) i\\
    =& \smashoperator{\sum_{(x,y,t,\hat y,i) \in \operatorname{supp}(X, Y, T, \widehat Y^\beta, I^\beta)}} \Pr(X{=}x,Y{=}y,T{=}t, \widehat Y^{\beta}{=}\hat y, \hat I^{\pi_{\theta}}{=}i|C) i
    \label{eqn: switch support}.
\end{align}
Finally, notice that by Assumption~\ref{ass: support}, $\text{supp}(\widehat Y^{\pitheta}) \subseteq \text{supp}(\widehat Y^\beta)$, and so $\text{supp}(I^{\pitheta}) \subseteq \text{supp}(I^\beta)$. So, we can rewrite~\eqref{eqn: switch support} as
\begin{equation}
    \smashoperator{\sum_{(x,y,t,\hat y,i) \in \operatorname{supp}(X, Y, T, \widehat Y^{\pitheta}, I^{\pitheta})}} \Pr(X{=}x,Y{=}y,T{=}t, \widehat Y^{\pi_{\theta}}{=}\hat y, \hat I^{\pi_{\theta}}{=}i|C) i.
\end{equation}
By the definition of expectation, this is equivalent to $\mathbf{E}\left [ I^{\pi_{\theta}} | C\right].$ Therefore, we have shown that $\mathbf{E}[\hat I^{\pi_{\theta}}|C]\!=\! \mathbf{E}[I^{\pi_{\theta}}|C]$.  
\end{proof}


\section{Bounds on delayed impact using Hoeffding's inequality}
\label{app: bound di using hoeff}
This section focuses on how one can use the unbiased estimates of $g(\theta)$ together with
Hoeffding's inequality \citep{hoeffding1963probability} to derive high-confidence upper bounds on $g(\theta)$. 
Given a vector of $m$ i.i.d.~samples $(Z_i)^m_{i=1}$ of a random variable $Z$, let $\bar Z = \frac{1}{m}\sum^m_{i=1} Z_i$ be the sample mean, and let $\delta \in (0,1)$ be a confidence level. 
\begin{property}[Hoeffding's Inequality]
    \label{prop: hoeffding}
    If $\Pr(Z \in [a,b])=1$, then 
    \begin{equation}
        \Pr \left(\mathbf E[Z_i] \geq \bar Z - (b-a) \sqrt{\frac{\ln (1/\delta)}{2m}}\right) \geq 1-\delta.
    \end{equation}
    \emph{\textbf{Proof.}} See the work of \citet{hoeffding1963probability}. \qed
\end{property}
Property~\ref{prop: hoeffding} can be used to obtain a high-confidence upper bound for the mean of $Z$:
\begin{equation}
    U_{\texttt{Hoeff}}(Z_1, Z_2, ..., Z_m) \coloneqq \bar Z + (b{-}a)\sqrt{\frac{\log(1/\delta)}{(2m)}}.
\end{equation}
Let $\hat g$ be a vector of i.i.d.~and unbiased estimates of $g(\theta)$. Once these are procured (using importance sampling as described in Section~\ref{sec: deriving estimates of delayed impact}), they can be provided to $U_\texttt{Hoeff}$ to derive a high-confidence upper bound on $g(\theta)$: 
\begin{equation}
    \Pr\left(\tau - \mathbf E\left [\hat I^{\pi_{\theta}} \middle | c(X, Y, T) \right] \leq U_\texttt{Hoeff}(\hat g)\right) \geq 1-\delta.
\end{equation}

Notice that using Hoeffding's inequality to obtain the upper bound requires the assumption that $\hat g$ is bounded. 

\section{Extensions of \algname}
\label{app: extensions}
In this section we discuss how \algname can be extended to provide similar high-confidence guarantees for the regression setting, for the classification setting with static fairness constraints, and for definitions of delayed impact beyond the form assumed in~\eqref{eqn: delayed impact objective}.
%

\subsection{Delayed-impact fairness guarantees in the regression setting}

In our problem setting, we study fairness with respect to delayed impact in the classification setting, in which the labels $Y$ produced by a model are discrete. However, our method can also be applied in the regression setting, where a (stochastic) regression model produces \emph{continuous} predictions $Y$, instead of discrete labels. To use \algname in this setting, one may adapt  Algorithm~\ref{alg: cost function} so that it uses a loss function suitable for regression; e.g., sample mean squared error. Furthermore, notice that the importance sampling technique described in Section~\ref{sec: deriving estimates of delayed impact} is still applicable in the regression setting, requiring only minor changes so that it can be used in such a continuous setting.
In particular, the importance sampling technique we described can be adapted by replacing summations with integrals, probability mass functions with probability density functions, and probabilities with probability densities. By doing so, all results presented in our work (e.g., regarding the unbiasedness of the importance sampling estimator) carry to the continuous case. Notice, finally, that in order to apply \algname in the regression setting, the behavior model, $\beta$, and the new candidate model, $\pi_\theta$, must be \textit{stochastic} regression models---this is similar to the assumption we made when addressing the classification setting (see the discussion in Section~\ref{sec: problem statement}).

\subsection{Enforcing static fairness constraints or constraints on model performance}
\label{app: enforcing static fairness}
In Section~\ref{sec: delayed impact in fair classification}, we showed how users can construct and enforce delayed-impact fairness constraints.
However, users might also be interested in simultaneously enforcing additional behavior---for instance, enforcing \emph{static} fairness constraints or constraints on the primary objective (i.e., on the classification or regression performance).
Assume, for example, that the bank from our running example has constructed a delayed-impact fairness constraint of the form in~\eqref{eqn: delayed impact objective}, and that it is \emph{also} interested in specifying an additional constraint that lower-bounds the model's performance in terms of accuracy. This could be represented by the following objective:
\begin{equation}
    g_\text{PERF}(\theta) = \tau_\text{PERF} - \E[\text{ACC}_\theta],
\end{equation}
%
%
where $\tau_\text{PERF} \in (0,1)$ represents the minimum required accuracy, and ACC$_\theta$ is a random variable representing the accuracy of the model $\theta$, i.e., the fraction of predictions that are correct given a random dataset $D$.

Continuing the example, assume that the bank would also like to enforce a \textit{static} fairness constraint: \textit{false positive error rate balance} (FPRB)~\citep{chouldechova2017fair}, also known as \textit{predictive equality}~\citep{corbett2017algorithmic}. 
Recall that a classifier satisfies an approximate version of FPRB if the absolute difference between the false positive rates of two demographic groups of loan applicants, A and B, is below some user-specified threshold. The bank could specify this static fairness objective as:
\begin{equation}
    g_\text{FPRB} = \left | \E\left[\widehat Y^{\pi_\theta} = 1 \middle | Y= 0, T = A \right] -  \E\left[\widehat Y^{\pi_\theta} = 1 \middle | Y= 0, T = B \right] \right | - \epsilon_\text{FPRB},
\end{equation}
where $\epsilon_\text{FPRB} \in (0,1)$ is the threshold of interest.

Notice that---unlike delayed-impact objectives---$g_\text{PERF}$ and $g_\text{FPRB}$ can be directly computed using only labels $Y$, predictions $\widehat Y^{\pi_\theta}$, and the sensitive attribute $T$; that is, information already available in the dataset or directly obtained from the model. 
In other words, the importance sampling method introduced in Section~\ref{sec: delayed impact in fair classification} would not be needed to obtain high-confidence bounds for these metrics. 
\citet{metevier2019offline} and \citet{thomas2019preventing} present methods to compute high-confidence upper bounds on static fairness constraints and constraints on performance. Notice that $g_\text{PERF}$ and $g_\text{FPRB}$ are just examples of this type of constraints; the techniques introduced by \citet{metevier2019offline} and \citet{thomas2019preventing} are applicable to more general objectives and constraints. We refer the reader to their work for more details.

To conclude our discussion of this example, notice that---once computed---high-confidence upper bounds on $g_\text{PERF}$ and $g_\text{FPRB}$ may be used to determine whether a candidate solution should be returned. Similar to line 11 of Algorithm~\ref{alg: main algorithm multiple constraints}, if all computed upper bounds (with respect to the accuracy objective, the predictive equality objective, and the delayed-impact objectives) are less than or equal to zero, then the candidate solution should be returned. Otherwise, \NSF
~should be returned.

\subsection{Alternative definitions of delayed impact}
Until now, we have assumed that the delayed-impact objective takes the form of~\eqref{eqn: delayed impact objective}.
%
%
Below, we discuss how users of \algname may construct other definitions of delayed impact, and how our formulation of delayed impact (shown in \eqref{eqn: delayed impact objective}) is related to the definitions introduced in the work of~\citet{liu2018delayed}.

\paragraph{Connections to the work of~\citet{liu2018delayed}.}
Our DI objective (1)~has the form $g(\theta) \coloneqq \tau - \mathbf E\left[I^{\pi_\theta} \middle | c(X, Y, T) \right]$.
This is similar to \emph{long-term improvement}, one of the notions of delayed impact introduced by \citet{liu2018delayed}.
Specifically, \citet{liu2018delayed} define long-term improvement as $\Delta \mu_j > 0$, where for group $j$, $\Delta\mu_j$ is the difference between the DI induced by a previously-deployed model and a new model. 
In their work, \citet{liu2018delayed} consider DI to be credit score. 
To enforce this type of long-term improvement in our framework, we can set $\tau$ to be group $j$'s average credit score under the current model (i.e., under the behavior model,  $\beta$) and $\mathbf E\left[I^{\pi_\theta} \middle | T=j \right]$ to be the expected credit score of group $j$ under the new model, $\pi_\theta$. Then, $\Delta \mu_j = \mathbf E\left[I^{\pi_\theta} \middle | T=j \right] - \tau$.
In our framework, a model is fair if $g(\theta) \leq 0$. Setting $g(\theta) = \tau - \mathbf E\left[I^{\pi_\theta} \middle | T=j \right]$ implies that the model is fair (i.e., $g(\theta) \leq 0$) if and only if $\theta$ leads to long-term improvement (i.e., iff $\Delta \mu_j > 0$). 

We can also use \algname to enforce constraints similar to the remaining definitions of delayed impact introduced by~\citet{liu2018delayed}; e.g., long-term decline ($\Delta\mu_j <0$) and long-term stagnation ($\Delta\mu_j=0$). 
Notice that long-term decline has a form similar to long-term improvement: $\Delta\mu_j < 0 \implies g(\theta) =  \mathbf E\left[I^{\pi_\theta} \middle | T=j \right] - \tau$. 
Alternatively, to enforce long-term stagnation users can set $g(\theta) = \big | \tau - E\left[I^{\pi_\theta} \middle | T=j \right]\big |$. Moreover, to enforce \textit{approximate} long-term stagnation, i.e., $| \Delta \mu | < \epsilon$, for some non-negative threshold $\epsilon$, users may set $g(\theta) = \big | \tau - \E\left[I^{\pi_\theta} \middle | T=j \right]\big | - \epsilon$.

Finally, notice that the definitions of long-term decline and stagnation do not have the same form as~\eqref{eqn: delayed impact objective}; nonetheless, these definitions \emph{can} be enforced using the methods introduced in our work---we discuss how to achieve this in the next section.

\paragraph{Enforcing general definitions of delayed impact.} 
%
%
%
To enforce constraints beyond~\eqref{eqn: delayed impact objective}, one can combine the importance sampling technique introduced in Section~\ref{sec: delayed impact in fair classification} with techniques presented in the work of~\citet{metevier2019offline}. 
Assume, for example, that the bank in our running example is interested in enforcing that the resulting expected delayed impact of a classifier's predictions is approximately equal for loan applicants of group $A$ and group $B$. This can be represented by the DI objective $g_\text{DI}(\theta) = \big | \E\left[I^{\pi_\theta} \middle | T=A \right] - \E\left[I^{\pi_\theta} \middle | T=B \right]\big | - \epsilon$. 
To enforce the DI constraint considered in this paper (i.e., the one shown in \eqref{eqn: delayed impact constraint}) on the more general types of DI objectives discussed in this appendix (such as $g_\text{DI}$), one may combine the techniques we introduced in this paper and the bound-propagation methods introduced by~\citet{metevier2019offline}. At a high-level, \algname would, in this case, first compute (as before) unbiased estimates of $\E\left[I^{\pi_\theta} \middle | T=A \right]$ and $\E\left[I^{\pi_\theta} \middle | T=B \right]$ using the importance sampling technique described in Section~\ref{sec: deriving estimates of delayed impact}. Then, it would use the bound-propagation methods introduced by ~\citet{metevier2019offline} to obtain high-confidence upper bounds on $g_\text{DI}(\theta)$. Notice that the discussion above corresponds to just one example of how to deal with alternative delayed-impact objective definitions; in this particular example, $g_\text{DI}$. The same general idea and techniques can, however, also be used to deal with alternative definitions of DI objectives that users of \algname may be interested in.\footnote{This statement holds assuming that the DI objective of interest satisfies the requirements for the bound-propagation technique to be applicable; for example, that the DI objective can be expressed using elementary arithmetic operations (e.g., addition and subtraction) over \textit{base variables} for which we know unbiased estimators~\citep{metevier2019offline}. In the case of the DI objectives discussed in this paper, for instance, we can obtain unbiased estimates of the relevant quantities using importance sampling, as discussed in Section~\ref{sec: deriving estimates of delayed impact}.} All other parts of the algorithm would remain the same---e.g., the algorithm would still split the dataset into two, identify a candidate solution, and check whether it passes the fairness test.

\paragraph{Beyond conditional expectation.}
In Section~\ref{sec: problem statement}, we assume that $g$ is defined in terms of the conditional expected value of the delayed-impact measure. 
However, other forms of fairness metrics might be more appropriate for different applications. 
For example, conditional value at risk \citep{keramati2020being} might be appropriate for risk-sensitive applications, and the median might be relevant for applications with noisy data~\citep{altschuler2019best}.
%
%
\citet{chandak2021universal} introduce off-policy evaluation methods that produce estimates and high-confidence bounds for different distributional parameters of interest, including value at risk, conditional value at risk, variance, median, and interquantile range. 
These techniques can also be combined with ours to obtain high-confidence upper bounds for metrics other than the conditional expected value of $I^{\pi_\theta}$.

\section{Proof of Theorem~\ref{thm: fairness guarantee}: Fairness Guarantee}
\label{app: fairness guarantee}

This section proves Theorem~\ref{thm: fairness guarantee}, which is restated below, along with the relevant assumptions.

\begin{tcolorbox}
\textbf{Assumption~\ref{ass: markov property}:} A model's prediction $\widehat Y^{\pitheta}$ is conditionally independent of $Y$ and $T$ given $X$. That is, for all $x,t, y$, and $\hat y$,
\begin{equation}
    \Pr(\widehat Y^\pitheta {=} \hat y | X{=}x, Y{=}y, T{=}t) = \Pr(\widehat Y^\pitheta {=} \hat y | X{=}x).
\end{equation}

\textbf{Assumption~\ref{ass: switch prediction}:} For all $x$, $y$, $t$, $\hat y$, $i$,\\
\begin{equation}
    \Pr(I^{\beta}{=}i|X{=}x,Y{=}y,T{=}t,\widehat Y^{\beta}{=}\hat y)
    =  \Pr(I^{\pi_{\theta}}=i|X{=}x,Y{=}y,T{=}t,\widehat Y^{\pi_{\theta}}{=}\hat y).
\end{equation}
    
\textbf{Assumption~\ref{ass: support} (Support):} 
For all $x$ and $y$, $\pi_{\theta}(x,y) > 0$ implies that $\beta(x,y) > 0$.

\textbf{Assumption~\ref{ass: restrict feasible set}:}  Every $\theta \in \Theta$ satisfies Assumption~\ref{ass: support}.

\textbf{Assumption~\ref{ass: estimates match bound}:} If \texttt{Bound} is \texttt{Hoeff}, then for all $j \in \{1, ..., k\}$, each estimate in $\hat g_j$ is bounded in some interval $[a_j, b_j]$. If \texttt{Bound} is \texttt{ttest}, then each Avg($\hat g_j$) is normally distributed. 

\end{tcolorbox}

\begin{tcolorbox}
\textbf{Theorem~\ref{thm: fairness guarantee}:} Let $(g_j{})^k_{j=1}$ be a sequence of DI constraints, where $g_j: \Theta \rightarrow \mathbb R$, and let $(\delta_j)^k_{j=1}$ be a corresponding sequence of confidence levels, where each $\delta_j \in (0,1)$. If Assumptions~\ref{ass: markov property}, \ref{ass: switch prediction}, \ref{ass: restrict feasible set}, and \ref{ass: estimates match bound} hold, then for all $j \in \{1, ..., k\}$,
\begin{equation}
    \Pr(g_j(a(D)) \leq 0) \geq 1-\delta_j. 
\end{equation}
\end{tcolorbox}

We first provide three lemmas that will be used when proving Theorem~\ref{thm: fairness guarantee}. 
\begin{lemma}
\label{lem: g estimates are identically distributed}
Let $\hat g_j$ be the estimates of $g$ constructed in Algorithm~\ref{alg: main algorithm multiple constraints}, and let $D_{f_c}$ be a subdataset of $D_f$ such that a data point $(X, Y, T, \widehat Y^\beta, I^\beta)$ is only in $D_{f_c}$ if $c(X, Y, T)$ is true. 
Then, for all $\theta\in\Theta$, the elements in $\hat g_j$ are $i.i.d.$ samples from the conditional distribution of $\hat g_j$ given $c(X,Y,T)$.
\end{lemma}
\begin{proof}
    To obtain $\hat g_j$, each data point in $D_{f_c}$ is transformed into an estimate of $g(\theta)$ using the importance sampling estimate $\tau - \frac{\pi_{\theta}(X, \widehat Y^\beta)}{\beta(X, \widehat Y^\beta)}I^\beta$ (Algorithm \ref{alg: main algorithm multiple constraints}, lines 5--8). 
Since each element of $\hat g_j$ is computed from a single data point in $D_{f_c}$, and the points in $D_{f_c}$ are conditionally independent given $c(X,Y,T)$, it follows that each element of $\hat g_j$ is conditionally independent given $c(X,Y,T)$. 
So, each element of $\hat g_j$ can be viewed as an i.i.d.~sample from the conditional distribution of $\hat g_j$ given $c(X,Y,T)$. 
\end{proof}

\begin{lemma}
\label{lem: g estimates are unbiased}
Let $\hat g_j$ be the estimates of $g$ constructed in Algorithm~\ref{alg: main algorithm multiple constraints}. If Assumptions~\ref{ass: markov property}, \ref{ass: switch prediction}, and~\ref{ass: restrict feasible set} hold, then for all $\theta\in\Theta$, each element in $\hat g_j$ is an unbiased estimate of $g_j(\theta)$.
\end{lemma}
\begin{proof}
     We begin by considering the expected value of any element in $\hat g_j$:
\begin{align}
    \E\left[\tau - \frac{\pi_{\theta}(X, \widehat Y^\beta)}{\beta(X, \widehat Y^\beta)}I^\beta \middle | c(X,Y,T)\right ]
    =& \tau - \E\left[\frac{\pi_{\theta}(X, \widehat Y^\beta)}{\beta(X, \widehat Y^\beta)}I^\beta \middle | c(X,Y,T) \right]\\
    =& \tau - \E\left[\hat I^{\pi_{\theta}} \middle | c(X,Y,T) \right]\label{eqn: sub defn of importance weight}\\
    =& \tau - \E\left[I^{\pi_{\theta}} \middle | c(X,Y,T) \right] \label{eqn: sub use of unbiased theorem}\\
    =& g_j(\theta).
    \end{align}
Expression~\eqref{eqn: sub use of unbiased theorem} follows from  Theorem~\ref{thm: is known behavior model}, which relies on Assumptions~\ref{ass: markov property}, \ref{ass: switch prediction}, and \ref{ass: restrict feasible set}. Therefore, for all $\theta\in\Theta$, the elements of $\hat g_j$ are unbiased estimates of $g_j(\theta)$.
\end{proof}

Let $\theta_c$ be the model returned by candidate selection in Algorithm~\ref{alg: main algorithm multiple constraints} (line 2), and let $U_j$ be the value of $U$ at iteration $j$ of the for loop (lines 4--10). 
\begin{lemma}
\label{lem: delta-bound on error}
If Assumptions~\ref{ass: markov property}, \ref{ass: switch prediction}, \ref{ass: restrict feasible set}, and \ref{ass: estimates match bound} hold, then the upper bounds $U_j$ calculated in Algorithm~\ref{alg: main algorithm multiple constraints} satisfy $\forall j \in \{1, ..., k\}$, $\Pr(g_j(\theta_c) > U_j) \leq \delta_j$.
\end{lemma}
\begin{proof}
We begin by noting that by Lemma~\ref{lem: g estimates are identically distributed}, the data points used to construct each $(1-\delta_j)$-probability bound, i.e., the data points in each $\hat g_j$, are (conditionally) i.i.d.
Because $\theta_c \in \Theta$, by Lemma~\ref{lem: g estimates are unbiased} (which uses Assumptions~\ref{ass: markov property}, \ref{ass: switch prediction}, and~\ref{ass: restrict feasible set}), we know that each element in $\hat g_j$ is an unbiased estimate of $g_j(\theta_c)$. 
Therefore, Hoeffding's inequality or Student's $t$-test can be applied to random variables that are (conditionally) i.i.d.\footnote{Samples that are conditionally i.i.d.~given some event $E$ can be viewed as i.i.d.~samples from the conditional distribution. Applying the confidence intervals to these samples therefore provides high-confidence bounds on the \emph{conditional} expected value given the event $E$, which is precisely what we aim to bound.} and unbiased estimates of $g_j(\theta_c)$. Moreover, under Assumption~\ref{ass: estimates match bound}, when \texttt{Bound} is \texttt{Hoeff}, the requirements of Hoeffding's inequality are satisfied (Property~\ref{prop: hoeffding}), and when \texttt{Bound} is \texttt{ttest}, the requirements of Student's $t$-test are satisfied (Property~\ref{prop: student's ttest}). 
Therefore, the upper bounds calculated in Algorithm~\ref{alg: main algorithm multiple constraints} satisfy $\Pr(g_j(\theta_c) > U_j) \leq \delta_j$.
\end{proof}

\paragraph{Proof of Theorem~\ref{thm: fairness guarantee}}

\begin{proof} To show Theorem~\ref{thm: fairness guarantee}, we prove the contrapositive, i.e., $ \forall j \in \{1, ..., k\}, \Pr(g_j(a(D)) > 0) \leq \delta_j$.  

Consider the event $\forall j \in \{1, ..., k\}, g_j(a(D)) > 0$. When this event occurs, it is always the case that $a(D) \neq \NSF$ (by definition, $g(\NSF) = 0$).
That is, a nontrivial solution was returned by the algorithm, and for all $j$, $U_j \leq 0$  (line $11$ of Algorithm~\ref{alg: main algorithm multiple constraints}). 
Therefore, \eqref{eqn: positive aD implies U <= 0} (shown below) holds.  
\begin{align}
    \Pr(g_j(a(D)>0) 
    =& \Pr(g_j(a(D))>0, U_j \leq 0)\label{eqn: positive aD implies U <= 0}\\
    \leq& \Pr(g_j(a(D)) > U_j)\label{eqn: joint implies g > U}\\
    =& \Pr(g_j(\theta_c) > U_j) \label{eqn: substitute nontrivial solution}\\
    \leq& \delta_j.\label{eqn: lemma}
\end{align}
Expression~\eqref{eqn: joint implies g > U} is a result of the fact that the joint event in~\eqref{eqn: positive aD implies U <= 0} implies the event $(g_j(a(D)) > U_j)$. 
We substitute $\theta_c$ for $a(D)$ in~\eqref{eqn: substitute nontrivial solution} because the event $\forall j \in \{1, ..., k\}, g_j(a(D)) > 0$ implies that a nontrivial solution, or a solution that is not \NSF, was returned: $a(D) = \theta_c$. 
Lastly,~\eqref{eqn: lemma} follows from Lemma~\ref{lem: delta-bound on error}. 
%
%
%
This implies that $\Pr(g_j(a(D)>0) \leq \delta_j \ \forall j \in \{1, ..., k\}$, completing the proof. 

\end{proof}

\section{Proof of Theorem~\ref{thm: consistency}}
\label{app: consistency proof}

This section proves Theorem~\ref{thm: consistency}, restated below. 
\citet{metevier2019offline} provide a similar proof for a Seldonian contextual bandit algorithm, which we adapt to our Seldonian classification algorithm. 
Extending their proof to our setting involves the following minor changes:
\begin{enumerate}
    \item Changes related to the output of the function used to calculate the utility of a solution: \citet{metevier2019offline} consider a utility function that returns the sample reward of a policy. Instead, our utility function (Algorithm~\ref{alg: cost function}) outputs the sample loss of a model.
    
    \item Changes due to the form of the fairness constraints: The form of our delayed-impact constraint differs from the more general form of the fairness constraints considered by~\citet{metevier2019offline}. This results in a simplified argument that our algorithm is consistent. 
\end{enumerate}

Rather than reword their proof with these minor changes, below we provide their proof with these minor changes incorporated. 

\begin{tcolorbox}
\textbf{Theorem~\ref{thm: consistency}:} If Assumptions~\ref{ass: markov property}--\ref{ass: consistent loss estimator} hold, then $\lim_{n\rightarrow \infty} \Pr(a(D) \neq \NSF, \ g(a(D)) \leq 0\big) = 1$.
\end{tcolorbox}
We begin by providing definitions and assumptions necessary for presenting our main result. To simplify notation, we assume that there exists only a single delayed-impact constraint and note that the extension of this proof to multiple delayed-impact constraints is straightforward.

Recall that the logged data, $D$, is a random variable. To further formalize this notion, let $(\Omega, \Sigma, p)$ be a probability space on which all relevant random variables are defined, and let $D_n: \Omega \rightarrow \mathcal D$ be a random variable, where $\mathcal D$ is the set of all possible datasets and $D_n = D_c \cup D_f$. 
We will discuss convergence as $n \rightarrow \infty$. $D_n(\omega)$ is a particular sample of the entire set of logged data with $n$ data points, where $\omega \in \Omega$.

\begin{definition}[Piecewise Lipschitz continuous]
\label{defn: Piecewise Lipshitz continuity}
We say that a function $f:M\rightarrow \mathbb{R}$ on a metric space $(M,d)$ is piecewise Lipschitz continuous with Lipschitz constant $K$ and with respect to  a countable partition, $\{M_1, M_2, ...\}$, of $M$ if $f$ is Lipschitz continuous with Lipschitz constant $K$ on all metric spaces in $\{(M_i,d)\}^{\infty}_{i=1}$.
\end{definition}
\begin{definition}[$\delta$-covering]
\label{defn: delta covering}
If $(M,d)$ is a metric space, a set $X \subseteq M$ is a $\delta$-covering of $(M,d)$ if and only if $\max\limits_{y\in M} \min\limits_{x \in X} d(x,y) \leq \delta$.
\end{definition}
Let $\hat c(\theta, D_c)$ denote the output of a call to Algorithm~{\ref{alg: cost function}}, and let $c(\theta) \coloneqq \ell_{\max} + g(\theta)$.
The next assumption ensures that $c$ and $\hat c$ are piecewise Lipschitz continuous. Notice that the $\delta$-covering requirement is straightforwardly satisfied if $\Theta$ is countable or $\Theta \subseteq \mathbb R^m$ for any positive natural number $m$.
\begin{assumption}
\label{ass: piecewise Lipschitz}
The feasible set of policies, $\Theta$, is equipped with a metric, $d_{\Theta}$, such that for all $D_c(\omega)$ there exist countable partitions of $\Theta$, $\Theta^c = \{\Theta^c_1, \Theta^c_2, ... \}$, and $\Theta^{\hat{c}} = \{\Theta^{\hat{c}}_1, \Theta^{\hat{c}}_2, ...\}$, where $c(\cdot)$ and $\hat{c}(\cdot,D_c(\omega))$ are piecewise Lipschitz continuous with respect to $\Theta^c$ and $\Theta^{\hat{c}}$ respectively with Lipschitz constants $K$ and $\hat{K}$. Furthermore, for all $i \in \mathbb N_{>0}$ and all $\delta > 0$ there exist countable $\delta$-covers of $\Theta^c_i$ and $\Theta^{\hat{c}}_i$.
\end{assumption}
Next, we assume that a fair solution, $\theta^\star$, exists such that $g(\theta^\star)$ is not precisely on the boundary of fair and unfair. 
This can be satisfied by solutions that are arbitrarily close to the fair-unfair boundary.
\begin{assumption}
\label{ass: existence of fair policy}
    There exists an $\epsilon > \xi$ and a $\theta^{\star} \in \Theta$ such that $g(\theta^{\star}) \leq -\epsilon$.
\end{assumption}
Next, we assume that the sample loss, $\hat \ell$, converges almost surely to $\ell$, the actual expected loss.
\begin{assumption}
\label{ass: consistent loss estimator}
$\forall \theta \in \Theta$, $\hat\ell(\theta,D_c) \overset{\text{a.s.}}{\longrightarrow} \ell(\theta)$. 
\end{assumption}

We prove Theorem~\ref{thm: consistency} by building up properties that culminate with the desired result, starting with a variant of the strong law of large numbers: 
\begin{property}[Khintchine Strong Law of Large Numbers]
\label{prop: law of large numbers}
Let $\{X_{\iota}\}_{i=1}^\infty$ be independent and identically distributed random variables. Then $(\frac{1}{n}\sum_{i=1}^n X_{\iota})_{n=1}^\infty$ is a sequence of random variables that converges almost surely to $\mathbf{E}[X_1]$, if $\mathbf{E}[X_1]$ exists, i.e., $\frac{1}{n}\sum_{i=1}^n X_{\iota} \overset{\text{a.s.}}{\longrightarrow} \mathbf{E}[X_1]$.
\end{property}
\begin{proof}
See Theorem 2.3.13 of \citet{Sen1993}.
\end{proof}

Next, we show that the average of the estimates of $g(\theta)$ converge almost surely to $g(\theta)$:
\begin{property}
    \label{prop: g estimates converge to g}
    If Assumptions~\ref{ass: markov property}, \ref{ass: switch prediction}, and~\ref{ass: restrict feasible set} hold, then $\forall \theta\in\Theta, \operatorname{Avg}(\hat g) \xrightarrow{\text{a.s.}}g(\theta)$. 
\end{property}
\begin{proof}
    %
    Recall that given Assumptions~\ref{ass: markov property}, \ref{ass: switch prediction}, and~\ref{ass: restrict feasible set}, Lemmas~\ref{lem: g estimates are identically distributed} and~\ref{lem: g estimates are unbiased} hold, i.e., estimates in $\hat g$ are i.i.d., and each estimate in $\hat g$ is an unbiased estimate of $g(\theta)$. 
    Also, recall that if $n_{\hat g}$ is the number of elements in $\hat g$, $\operatorname{Avg}(\hat g) \coloneqq \frac{1}{n_{\hat g}}\sum_{i=1}^{n_{\hat g}}\hat g_i$. 
   Then, by Property~\ref{prop: law of large numbers} we have that $\operatorname{Avg}(\hat g) \xrightarrow{\text{a.s.}}g(\theta)$.
\end{proof}

In this proof, we consider the set $\bar{\Theta} \subseteq \Theta$, which contains all solutions that are not fair, and some that are fair but fall beneath a certain threshold: $\bar{\Theta} \coloneqq \{\theta \in \Theta: g(\theta) > -\xi/2\}$.
At a high level, we will show that the probability that the candidate solution, $\theta_c$, viewed as a random variable that depends on the candidate data set $D_c$, satisfies $\theta_c \not \in \bar{\Theta}$ converges to one as $n \rightarrow \infty$, and then that the probability that $\theta_c$ is returned also converges to one as $n \rightarrow \infty$. 

First, we will show that the upper bounds $U^+$ (constructed in candidate selection, i.e., Algorithm~\ref{alg: cost function}) and $U$ (constructed in the fairness test, i.e., Algorithm~\ref{alg: main algorithm}) converge to $g(\theta)$ for all $\theta \in \Theta$. 
To clarify notation, we write $U^+(\theta, D_c)$ and $U(\theta, D_f)$ to emphasize that each depends on $\theta$ and the datasets $D_c$ and $D_f$, respectively.
\begin{property}
\label{prop: bounds converge to g}
If Assumptions~\ref{ass: markov property}, \ref{ass: switch prediction}, \ref{ass: restrict feasible set}, and~\ref{ass: estimates match bound} hold, then for all $\theta \in \Theta$, $U^+(\theta, D_c) \xrightarrow{\text{a.s.}}g(\theta)$ and $U(\theta, D_f) \xrightarrow{\text{a.s.}}g(\theta)$. 
\end{property}
\begin{proof}
 Given Assumption~\ref{ass: estimates match bound}, Hoeffding's inequality and Student's $t$-test construct high-confidence upper bounds on the mean by starting with the sample mean of the unbiased estimates (in our case, $\operatorname{Avg}(\hat g))$ and then adding an additional term (a constant in the case of Hoeffding's inequality). 
    Thus, $U(\theta, D_f)$ can be written as $\operatorname{Avg}(\hat g)+Z_n$, where $Z_n$ is a sequence of random variables that converges (surely for Hoeffding's inequality, almost surely for Student's $t$-test) to zero. 
    So, $Z_n\overset{\text{a.s.}}{\longrightarrow}0$, and we need only show that $\operatorname{Avg}(\hat g)\overset{\text{a.s.}}{\longrightarrow}g(\theta)$, which follows from Assumptions~\ref{ass: markov property}, \ref{ass: switch prediction}, and Property~\ref{ass: restrict feasible set}. 
    We therefore have that  $U\overset{\text{a.s.}}{\longrightarrow} g(\theta)$.
    
    The same argument can be used when substituting $U^+(\theta,D_c)$ for $U(\theta,D_f)$. Notice that the only difference between the method used to construct confidence intervals in the fairness test (that is, $U^+$) and in Algorithm~\ref{alg: cost function} (that is, $U$) is the multiplication of $Z_n$ by a constant $\lambda$. 
    This still results in a sequence of random variables that converges (almost surely for Student's $t$-test) to zero.
\end{proof}

Recall that we define $\hat c(\theta, D_c)$ to be the output of Algorithm~\ref{alg: cost function}. Below, we show that given a fair solution $\theta^\star$ and data $D_c$, $\hat c(\theta^\star, D_c)$ converges almost surely to $\ell(\theta^\star)$, the expected loss of $\theta^\star$. 
\begin{property}
\label{prop: convergence of fair candidateValue}
If Assumptions~\ref{ass: markov property}, \ref{ass: switch prediction}, \ref{ass: restrict feasible set}, \ref{ass: estimates match bound}, \ref{ass: existence of fair policy}, and~\ref{ass: consistent loss estimator} hold, 
$\hat{c}(\theta^{\star}, D_c) \overset{\text{a.s.}}{\longrightarrow}$ $\ell(\theta^\star)$.
\end{property}
\begin{proof}
By Property~\ref{prop: bounds converge to g} (which holds given Assumptions~\ref{ass: markov property}, \ref{ass: switch prediction}, \ref{ass: restrict feasible set}, and~\ref{ass: estimates match bound}), we have that  $U^+$$(\theta^{\star}) \overset{\text{a.s.}}{\longrightarrow} g(\theta^{\star})$. By Assumption~\ref{ass: existence of fair policy}, we have that $g(\theta^{\star}) \leq -\epsilon$.
Now, let 
    \begin{equation}
    A=\{\omega \in \Omega : \lim_{n\to\infty}U^+(\theta^\star, D_c(\omega))=g(\theta^\star)\}.
    \end{equation}
     Recall that $U^+(\theta^\star, D_c)\overset{\text{a.s.}}{\longrightarrow}g(\theta^\star)$ means that $\Pr(\lim_{n\to\infty} U^+$$(\theta^\star, D_c)=g(\theta^\star))=1$. So, $\omega$ is in $A$ almost surely, i.e., $\Pr(\omega \in A)=1$. 
     Consider any $\omega \in A$. 
     From the definition of a limit and the previously established property that $g(\theta^\star) \leq -\epsilon$, we have that there exists an $n_0$ such that for all $n \geq n_0$, Algorithm~\ref{alg: cost function}  will return $\hat \ell(\theta^\star, D_c)$ (this avoids the discontinuity of the \texttt{if} statement in Algorithm~\ref{alg: cost function}
     for values smaller than $n_0$).  

     Furthermore, we have from Assumption~\ref{ass: consistent loss estimator} that $\hat \ell(\theta^\star,D_c)\overset{\text{a.s.}}{\longrightarrow} \ell(\theta^\star)$.
     Let 
    \begin{equation}
        B=\{\omega \in \Omega : \lim_{n\to\infty}\hat \ell(\theta^\star,D_c(\omega)) = \ell(\theta^\star) \}.
    \end{equation}
     From Assumption~\ref{ass: consistent loss estimator}, we have that $\omega$ is in $B$ almost surely, i.e., $\Pr(\omega \in B)=1$, and thus by the countable additivity of probability measures, $\Pr(\omega \in (A \cap B))=1$. 
     
     Consider now any $\omega \in (A \cap B)$. We have that for sufficiently large $n$, Algorithm~\ref{alg: cost function}
     will return $\hat \ell(\theta^\star,D_c)$ (since $\omega \in A$), and further that $\hat \ell
     (\theta^\star,D_c) \to \ell(\theta^\star)$ (since $\omega \in B$). 
     Thus, for all $\omega \in (A \cap B)$, the output of Algorithm~\ref{alg: cost function} converges to $\ell(\theta^\star)$, i.e., $\hat c(\theta^\star, D_c(\omega)) \to\ell(\theta^\star)$. Since $\Pr(\omega \in (A \cap B))=1$, we conclude that $\hat c(\theta^\star, D_c(\omega))  \overset{\text{a.s.}}{\longrightarrow}$$\ell(\theta^\star)$.
\end{proof} 

We have now established that the output of Algorithm~\ref{alg: cost function} converges almost surely to $\ell(\theta^\star)$ for the $\theta^\star$ assumed to exist in Assumption~\ref{ass: existence of fair policy}. 
We now establish a similar result for all $\theta \in \bar \Theta$---that the output of Algorithm~\ref{alg: cost function} converges almost surely to $c(\theta)$ (recall that $c(\theta)$ is defined as $\ell_{\max} + g(\theta)$).
\begin{property}
\label{prop: convergence of unfair candidateValue}
If Assumptions~\ref{ass: markov property}, \ref{ass: switch prediction}, \ref{ass: restrict feasible set}, and~\ref{ass: estimates match bound} hold, then for all $\theta \in \bar \Theta, \ \hat{c}(\theta, D_c) \overset{\text{a.s.}}{\longrightarrow} c(\theta)$.
\end{property}
\begin{proof}
By Property~\ref{prop: bounds converge to g} (which holds given Assumptions~\ref{ass: markov property}, \ref{ass: switch prediction}, \ref{ass: restrict feasible set}, and~\ref{ass: estimates match bound}), we have that $U^+(\theta, D_c)$$\overset{\text{a.s.}}{\longrightarrow}g(\theta)$. 
    If $\theta \in \bar \Theta$, then we have that $g(\theta) > -\xi/2$. 
    We now change the definition of the set $A$ from its definition in the previous property to a similar definition suited to this property. That is, let:
    \begin{equation}
    A=\{\omega \in \Omega : \lim_{n\to\infty}U^+(\theta, D_c(\omega))=g(\theta)\}.
    \end{equation}
    Recall that $U^+(\theta, D_c)$$\overset{\text{a.s.}}{\longrightarrow}g(\theta)$ means that $\Pr(\lim_{n\to\infty} U^+(\theta, D_c)=g(\theta))=1$. So, $\omega$ is in $A$ almost surely, i.e., $\Pr(\omega \in A)=1$. 
    Consider any $\omega \in A$. 
    From the definition of a limit and the previously established property that $g(\theta) > -\xi/2$, we have that there exists an $n_0$ 
    such that for all $n \geq n_0$ Algorithm~\ref{alg: cost function} will return $\ell_{\max} + U^+(\theta, D_c(\omega))$. 
    By Property~\ref{prop: bounds converge to g} (which holds given Assumptions~\ref{ass: markov property}, \ref{ass: switch prediction}, \ref{ass: restrict feasible set}, and~\ref{ass: estimates match bound}), $U^+(\theta,D_c(\omega))\overset{\text{a.s.}}{\longrightarrow} g(\theta)$. 
    So, for all $\omega \in A$, the output of Algorithm~\ref{alg: cost function} converges almost surely to $\ell_\text{max}+g(\theta)$; that is, $\hat c(\theta,D_c(\omega)) \overset{\text{a.s.}}{\longrightarrow}\ell_\text{max}+g(\theta)$, and since $c(\theta)=\ell_\text{max}+g(\theta)$, we therefore conclude that  $\hat c(\theta, D_c(\omega))  \overset{\text{a.s.}}{\longrightarrow} c(\theta)$. 
\end{proof}

By Property~\ref{prop: convergence of unfair candidateValue} and one of the common definitions of almost sure convergence, 
$$\forall \theta \in \bar \Theta, \forall \epsilon > 0, \Pr\Big( \lim\limits_{n \rightarrow \infty} \text{inf} \{\omega \in \Omega: | \hat c(\theta, D_n(\omega)) - c(\theta)| < \epsilon\} \Big) = 1.$$
Because $\Theta$ is not countable, it is not immediately clear that all $\theta \in \bar \Theta$ converge simultaneously to their respective $c(\theta)$. 
We show next that this is the case due to our smoothness assumptions.
\begin{property}
    \label{prop: simultaneous convergence}
    If Assumptions~\ref{ass: markov property}, \ref{ass: switch prediction}, \ref{ass: restrict feasible set}, \ref{ass: estimates match bound}, and~\ref{ass: piecewise Lipschitz} hold, then $\forall \epsilon' > 0$, 
    \begin{equation}
        \Pr\Big( \lim\limits_{n \rightarrow \infty} \emph{inf} \{\omega \in \Omega: \forall \theta \in \bar \Theta, |\hat c(\theta, D_c(\omega)) - c(\theta)| < \epsilon'\} \Big) = 1.
    \end{equation}
\end{property}
\begin{proof}
    Let C$(\delta)$ denote the union of all the points in the $\delta$-covers of the countable partitions of $\Theta$ assumed to exist by Assumption~\ref{ass: piecewise Lipschitz}. 
    Since the partitions are countable and the $\delta$-covers for each region are assumed to be countable, we have that C$(\delta)$ is countable for all $\delta$. 
    Then by Property \ref{prop: convergence of unfair candidateValue} (which holds given Assumptions~\ref{ass: markov property}, \ref{ass: switch prediction}, \ref{ass: restrict feasible set}, and~\ref{ass: estimates match bound}), for all $\delta$, we have convergence for all $\theta \in \emph{C}(\delta)$ simultaneously:
    \begin{equation}
        \label{eq: delta convergence}
        \forall \delta > 0, \forall \epsilon > 0, \Pr\Big( \lim\limits_{n \rightarrow \infty} \text{inf} \{\omega \in \Omega: \forall \theta \in \emph{C}(\delta), |\hat c(\theta, D_c(\omega)) - c(\theta)| < \epsilon\} \Big) = 1.
    \end{equation}
    Now, consider a $\theta \not\in \emph{C}(\delta)$. By Definition~\ref{defn: delta covering} and Assumption~\ref{ass: piecewise Lipschitz}, $\exists \ \theta' \in \bar \Theta^c_i, d(\theta, \theta') \leq \delta$.
    Moreover, because $c$ and $\hat c$ are Lipschitz continuous on $\bar \Theta^c_i$ and $\bar \Theta^{\hat c}_i$ (by Assumption~\ref{ass: piecewise Lipschitz}) respectively, we have that $|c(\theta) - c(\theta')| \leq K\delta$ and $|\hat c(\theta, D_c(\omega)) - \hat c(\theta', D_c(\omega))| \leq \hat K\delta$. 
    %
    %
    So, $|\hat c(\theta, D_c(\omega)) - c(\theta)| \leq 
    |\hat c(\theta,D_c(\omega)) - c(\theta')| + K\delta \leq
    |\hat c(\theta',D_c(\omega))-c(\theta')| + \delta(K + \hat K)$.
    This means that for all $\delta > 0$:
    $$\Big( \forall \theta \in \text{C}(\delta), |\hat c(\theta, D_c(\omega)) - c(\theta)| < \epsilon\Big) \implies \Big(\forall \theta \in \bar \Theta, |\hat c(\theta, D_c(\omega)) - c(\theta)| < \epsilon + \delta (K + \hat K)\Big).$$
    Substituting this into~\eqref{eq: delta convergence}, we get:
    $$\forall \delta > 0, \forall \epsilon > 0, \Pr\Big(\lim\limits_{n \rightarrow \infty} \text{inf} \{\omega \in \Omega: \forall \theta \in \bar \Theta, \ |\hat c(\theta, D_c(\omega)) - c(\theta)| < \epsilon + \delta (K + \hat K)\}\Big) = 1.$$
    Now, let $\delta \coloneqq \epsilon/(K + \hat K)$ and $\epsilon'=2\epsilon$. Thus, we have the following:
    $$\forall \epsilon' > 0, \Pr\Big( \lim\limits_{n \rightarrow \infty} \text{inf} \{\omega \in \Omega: \forall \theta \in \bar \Theta, |\hat c(\theta, D_c(\omega)) - c(\theta)| < \epsilon'\} \Big) = 1.$$
\end{proof}

So, given the appropriate assumptions, for all $\theta \in \bar \Theta$, we have that $\hat c(\theta, D_c(\omega)) \overset{\text{a.s.}}{\longrightarrow} c(\theta)$ and that $\hat c(\theta^\star, D_c(\omega)) \overset{\text{a.s.}}{\longrightarrow} \ell(\theta^\star$).
Due to the countable additivity property of probability measures and Property~\ref{prop: simultaneous convergence}, we have the following: 

\begin{equation}
\label{eq: limits capped}
\Pr\Big(\Big[ \forall \theta \in \bar \Theta, \lim\limits_{n \rightarrow \infty} \hat c(\theta, D_c(\omega)) = c(\theta)\Big], \ \Big[ \lim\limits_{n \rightarrow \infty} \hat c(\theta^\star, D_c(\omega)) = \ell(\theta^\star)\Big]\Big) = 1,
\end{equation}
where $\Pr(A,B)$ denotes the joint probability of $A$ and $B$.

Let $H$ denote the set of $\omega \in \Omega$ such that~\eqref{eq: limits capped} is satisfied. Note that $\ell_{\max}$ is defined as the value always greater than $\ell(\theta)$ for all $\theta \in \Theta$, and $g(\theta)\geq-\xi$ for all $\theta \in \bar \Theta$.
So, for all $\omega \in H$, for sufficiently large $n$, candidate selection will not define $\theta_c$ to be in $\bar \Theta$.
Since $\omega$ is in $H$ almost surely ($\Pr(\omega \in H)=1$), we therefore have that $\lim_{n \rightarrow \infty} \Pr(\theta_c \not\in \bar \Theta) = 1$.

The remaining challenge is to establish that, given $\theta_c \not \in \bar \Theta$, the probability that the fairness test returns $\theta_c$ rather than \NSF\ converges to one as $n \rightarrow \infty$. 
%
%
By Property~\ref{prop: bounds converge to g}, we have that $U(\theta_c,D_f)\overset{\text{a.s.}}{\longrightarrow} g(\theta_c)$. 
Furthermore, by the definition of $\bar \Theta$, when $\theta_c \not \in \bar \Theta$ we have that $g(\theta_c) < -\xi / 2$. 
So, $U(\theta_c, D_f)$ converges almost surely to a value less than $-\xi/2$. 
Since the fairness test returns $\theta_c$ rather than \NSF~if $U(\theta_c,D_f) \leq -\xi/4$ and $U(\theta_c, D_f)$ converges almost surely to a value less than $-\xi/2$, it follows that the probability that $U(\theta_c,D_f) \leq -\xi/4$ converges to one. 
Hence, given that $\theta_c \not \in \bar \Theta$, the probability that $\theta_c$ is returned rather than \NSF~converges to one.

We therefore have that \textbf{1)} the probability that  $\theta_c \not \in \bar \Theta$ converges to one as $n \to \infty$ and \textbf{2)} given that $\theta_c \not \in \bar \Theta$, the probability that $\theta_c$ is returned rather than \NSF~converges to one. 
Since $\theta_c \not \in \bar \Theta$ implies that $\theta_c$ is fair, these two properties imply that the probability that a fair solution is returned converges to one as $n \to \infty$. 

\section{A discussion on the intuition and implications of our assumptions}
\label{app: assumption intuitions} 
 
The theoretical results in this paper, which ensure \algname's convergence and high-confidence fairness guarantees, are based on Assumptions 1--8. In this section, we provide an intuitive, high-level discussion on the meaning and implications of these assumptions. Our goal is to show that they are standard in the machine learning literature and reasonable in many real-life settings.

\textbf{Assumption 1.} This assumes the commonly-occurring setting in which a classifier's prediction depends only on an input feature vector, $X$. It is a formal characterization of the classic supervised learning setting; that is, that machine learning models should predict the target variable of interest, $Y$, based only on an input feature vector, $X$. In other words, we are dealing with a standard classification problem.

\textbf{Assumption 2.} The delayed impact of a decision depends only on the decision itself, not on the machine learning algorithm used to make it. In our running example, for instance, it makes no difference whether an SVM or a neural network made a loan-repayment prediction; its delayed impact depends only on whether the person actually received the loan.

\textbf{Assumptions 3 and 4.} These assumptions can be trivially satisfied when our algorithm uses standard modern stochastic classifiers that place non-zero probability on all outputs; for instance, when using Softmax layers in a neural network. This assumption is common in the offline RL literature---for example, in methods that evaluate new policies given information from previously-deployed policies.

\textbf{Assumption 5.} Our algorithm uses standard statistical tools, common in the machine learning literature, to compute confidence bounds: Hoeffding's inequality and Student's t-test. Hoeffding's inequality can be applied under mild assumptions. In our running example, lending decisions made by a bank may have a delayed impact, e.g., on an applicant's future savings rate. Hoeffding's inequality holds if the bank knows the minimum and maximum savings rate possible (i.e., the DI is bounded in some interval $[a,b]$). Bounds produced by Student's t-test hold exactly if the sample mean is normally distributed, and in the limit (as the number of samples increases) if the sample mean follows a different distribution. With few samples, bounds based on Student's t-test may hold approximately. Despite this, their use remains effective and commonplace in the sciences, including, e.g., in high-risk medical research \citep{thomas2019preventing}.

\textbf{Assumption 6.} The cost function used to evaluate classifiers is smooth: similar classifiers have similar costs/performances. Smoothness assumptions of this type are common the machine learning literature. Also, each classifier can be described by a set of real-valued parameters ($\theta \in \mathbb R^m$), as is the case with all parametric supervised learning algorithms. 

\textbf{Assumption 7.} The space of classifiers is not degenerate: at least one fair solution exists such that if we perturb its parameters infinitesimally, it would not become arbitrarily unfair.

\textbf{Assumption 8.} The sample performance of a classifier converges to its true expected performance given enough data. This is similar to the usual assumption, e.g., in the regression setting, that a model's sample Mean Squared Error (MSE) converges to its true MSE given sufficient examples.

\section{Full algorithm}
\label{app: full algorithm}

\begin{algorithm}[tb]
\caption{\texttt{cost}}
\label{alg: cost function}
\textbf{Input}: \textbf{1)} the vector $\theta$ that parameterizes model $\pi$; \textbf{2)} $D_c = \{(X_i, Y_i, T_i, \widehat Y^\beta_i, I^\beta_i)\}^{m}_{i=1}$; 
\textbf{3)} confidence level $\delta$;
\textbf{4)} tolerance value $\tau$;
 \textbf{5)} the behavior model $\beta$; 
\textbf{6)} $\texttt{Bound} \in\{ \texttt{Hoeff}, \texttt{ttest}\}$; and 
\textbf{7)} the number of data points in $D_f$, denoted $n_{D_f}$.\\
\textbf{Output}: The cost of $\pi$.
\begin{algorithmic}[1] 
\STATE $\hat g \leftarrow \langle \ \rangle$
\FOR{$i\in\{1,..., m\}$}
\STATE \textbf{if} $c(X_i, Y_i, T_i)$ is \texttt{True} \textbf{then} $\hat g$.append$\left(\tau - \frac{\pi_\theta(X_i,\widehat
Y^\beta_i)}{\beta(X_i, \widehat Y^\beta_i)}I^\beta_i\right)$ \textbf{end if} 
\ENDFOR 
\STATE Let $\lambda = 2; \quad n_{\hat g} = \texttt{length}(\hat g)$
%
\STATE \textbf{if} {\texttt{Bound} is \texttt{Hoeff}} \textbf{then } 
\STATE $\quad a,b \leftarrow$ upper and lower bounds of $g$
\STATE $\quad U^{+} = \frac{1}{n_{\hat g}}\left( \sum^{n_{\hat g}}_{\iota=1} \hat g_\iota\right) + \lambda  (b{-}a)\sqrt{\frac{\log(1/\delta)}{(2n_{D_f})}}$ 
\STATE \textbf{else if} \texttt{Bound} is \texttt{ttest} \textbf{then} $U^{+} = \frac{1}{n_{\hat g}}\left( \sum^{n_{\hat g}}_{\iota=1} \hat g_\iota\right) + \lambda\frac{\sigma(\hat g)}{\sqrt{n_{D_f}}}t_{1-\delta, n_{D_f}-1}$ 
\STATE \textbf{end if}
\STATE $\ell_{\max} = \max_{\theta' \in \Theta} \hat \ell(\theta', D_c)$
\STATE \textbf{if} $U^{+} {\leq} {-}\frac{\xi}{4}$ \textbf{return} $\hat \ell(\theta,D_c)$ \textbf{else return}  $\left(\ell_{\max} + U^{+}\right)$
\end{algorithmic}
\end{algorithm}

\begin{algorithm}[tb]
\caption{\algname with Multiple Constraints}
\label{alg: main algorithm multiple constraints}
\textbf{Input}: \textbf{1)} dataset $D = \{(X_i, Y_i, T_i, \widehat Y^\beta_i, I^\beta_i)\}^n_{i=1}$;
\textbf{2)} the number of delayed-impact constraints we wish to satisfy, $k$; 
\textbf{3)} a sequence of Boolean conditionals $(c_j)_{j=1}^k$ such that for $j\in\{1,...,k\}$, $c_j(X_i, Y_i, T_i)$ indicates whether the event associated with the data point $(X_i, Y_i, T_i, \widehat Y_i^\beta, I^\beta_i)$ occurs;
\textbf{4)} confidence levels $\delta=(\delta_j)_{j=1}^k$, where each $\delta_j \in (0,1)$ corresponds to delayed-impact constraint $g_j$; %
\textbf{5)} tolerance values $\tau = (\tau_j)_{j=1}^k$, where each $\tau_j$ is the tolerance associated with delayed-impact constraint $g_j$;
\textbf{6)} the behavior model $\beta$; 
 and
\textbf{7)} an argument $\texttt{Bound} \in\{ \texttt{Hoeff}, \texttt{ttest}\}$ indicating which method for calculating upper bounds to use.
\\
\textbf{Output}: Solution $\theta_c$ or \NSF.
\begin{algorithmic}[1] 
\STATE $D_c, D_f \leftarrow \texttt{partition}(D)$
\STATE $\theta_c \leftarrow \arg\min_{\theta\in\Theta}  \texttt{cost}(\theta,D_c,k, \delta, \tau, \beta, \texttt{Bound}, \texttt{length}(D_f))$
\STATE $U \leftarrow \langle \ \rangle$
\FOR{$j \in \{1,...,k\}$}
\STATE $\hat g_j \leftarrow \langle \ \rangle$
\FOR{$i\in\{1,...,n\}$}
\STATE \textbf{if } $c_j(X_i, Y_i, T_i)$ is \texttt{True} \textbf{ then} $\hat g_j$.append$\left(\tau_j - \frac{\pi_{\theta_c}(X_i,\widehat
Y^\beta_i)}{\beta(X_i, \widehat Y^\beta_i)}I^\beta_i\right)$ \textbf{end if} 
\ENDFOR 

%
%
\STATE \textbf{if} \texttt{Bound} is \texttt{Hoeff} \textbf{then} $U.$append($U_\texttt{Hoeff} (\hat g_j)$) \textbf{else} $U.$append($U_{\texttt{ttest}}(\hat g_j)$) \textbf{end}
%
%
\ENDFOR
\STATE \textbf{if} $\forall j\in\{1, ..., k\}, U_j\leq 0$ \textbf{then return} $\theta_c$  \textbf{else return} \NSF\\
\end{algorithmic}
\end{algorithm}
\begin{algorithm}[tb]
\caption{\texttt{cost} with Multiple Constraints}
\label{alg: utility function multiple constraints}
\textbf{Input}: \textbf{1)} the vector $\theta$ that parameterizes a classification model $\pi$; \textbf{2)} candidate dataset $D_c = \{(X_i, Y_i, T_i, \widehat Y^\beta_i, I^\beta_i)\}^{m}_{i=1}$; 
\textbf{3)} the number of delayed-impact constraints we wish to satisfy, $k$; 
\textbf{4)} a sequence of Boolean conditionals $(c_j)_{j=1}^k$ such that for $j\in\{1,...,k\}$, $c_j(X_i, Y_i, T_i)$ indicates whether the event associated with the data point $(X_i, Y_i, T_i, \widehat Y_i^\beta, I^\beta_i)$ occurs;
\textbf{5)} confidence levels $\delta=\{\delta_j\}_{j=1}^k$, where each $\delta_j \in (0,1)$ corresponds with constraint $g_j$;
\textbf{6)} tolerance values $\tau = \{\tau_j\}_{j=1}^k$, where each $\tau_j$ is the tolerance associated with delayed-impact constraint $g_j$;
 \textbf{7)} the behavior model $\beta$; 
\textbf{8)} an argument $\texttt{Bound} \in\{ \texttt{Hoeff}, \texttt{ttest}\}$ indicating which method for calculating upper bounds to use; and 
\textbf{9)} the number of data points in dataset $D_f$, denoted $n_{D_f}$.\\
\textbf{Output}: The cost associated with classification model $\pi_\theta$.
\begin{algorithmic}[1] 
\FOR{$j \in \{1,...,k\}$}
\STATE $\hat g_j \leftarrow \langle \ \rangle$
\FOR{$i\in\{1, ..., m\}$}
\STATE \textbf{if } $c_j(X_i, Y_i, T_i)$ is \texttt{True} \textbf{ then} $\hat g_j$.append$\left(\tau_j - \frac{\pi_\theta(X_i,\widehat
Y^\beta_i)}{\beta(X_i, \widehat Y^\beta_i)}I^\beta_i\right)$ \textbf{end if}  
\ENDFOR 
\STATE Let $\lambda = 2; \quad n_{\hat g_j} = \texttt{length}(\hat g_j)$
\IF {\texttt{Bound} is \texttt{Hoeff}}
\STATE Let $a,b$ be the lower and upper bounds of $g_j$
\STATE $U_j^+ = \frac{1}{n_{\hat g_j}}\left( \sum^{n_{\hat g_j}}_{\iota=1} (\hat g_j)_\iota\right) + \lambda(b{-}a)\sqrt{\frac{\log(1/\delta_j)}{(2n_{D_f})}}$ 
\ELSIF{\texttt{Bound} is \texttt{ttest}}
\STATE $U_j^+ = \frac{1}{n_{\hat g_j}}\left( \sum^{n_{\hat g_j}}_{\iota=1} (\hat g_j)_\iota\right) + \lambda \frac{\sigma(\hat g_j)}{\sqrt{n_{D_f}}}t_{1-\delta_j, n_{D_f}-1}$ 
\ENDIF
\ENDFOR
\STATE $\ell_{\max} = \max_{\theta' \in \Theta} \hat \ell(\theta', D_c)$
\STATE \textbf{if} $\forall j \in \{1,...,k\}, U_j^+ \leq -\xi / 4 \,\,$ \textbf{then return} $\hat \ell(\theta,D_c)$ \textbf{else return} $\left(\ell_{\max} +  \sum_{j=1}^k U^{\text{inflated}}_j\right)$  
%
\end{algorithmic}
\end{algorithm}

Algorithm~\ref{alg: cost function} presents the cost function used in candidate selection (line 3 of Algorithm~\ref{alg: main algorithm}), where a strategy like the one used in the fairness test is used to calculate the cost, or utility, of a potential solution $\theta$.
Again, unbiased estimates of $g(\theta)$ are calculated, this time using dataset $D_c$ (lines 2--4).
%
%
Instead of calculating a high-confidence upper bound on $g(\theta)$ using $U_\texttt{Hoeff}$ or $U_\texttt{ttest}$, we calculate an \textit{inflated} upper bound $U^+$. Specifically, we inflate the width of the confidence interval used to compute the upper bound (lines 5--10).
This is to mitigate the fact that multiple comparisons are performed on the same dataset ($D_c$) during the search for a candidate solution (see line $3$ of Algorithm~\ref{alg: main algorithm}), which often leads candidate selection to overestimate its confidence that the solution it picks will pass the fairness test. 
Our choice to inflate the confidence interval in this way, i.e., considering the size of the dataset $D_f$ used in the fairness test and the use of scaling constant $\lambda$, is empirically driven and was first proposed for other Seldonian algorithms~\citep{thomas2019preventing}.

If $U^+ \leq -\xi/4$, a small negative constant, the cost associated with the loss of $\theta$ 
is returned. 
Otherwise, the cost of $\theta$ is defined as the sum of $U^+$ and the maximum loss that can be obtained on dataset $D_c$ (lines 11--12). 
This discourages candidate selection from returning models unlikely to pass the fairness test. 
%
We consider $-\xi/4$, instead of $0$ as the fairness threshold in \algname to ensure consistency. This is discussed in more detail in Appendix~\ref{app: consistency proof}. 

Algorithm~\ref{alg: main algorithm multiple constraints} shows \algname with multiple constraints. The changes relative to Algorithm \ref{alg: main algorithm} are relatively small: instead of considering only a single constraint, the fairness test loops over all $k$ constraints and only returns the candidate solution if all $k$ high-confidence upper bounds are at most zero. Similarly, the cost function, Algorithm \ref{alg: utility function multiple constraints}, changes relative to Algorithm \ref{alg: cost function} in that when predicting the outcome of the fairness test it includes this same loop over all $k$ delayed-impact constraints.

\section{Other experiments}
\label{app: experiments}

In all experiments, our implementation of \algname used CMA-ES~\citep{hansen2001completely} to search over the space of candidate solutions and the $\texttt{ttest}$ concentration inequality. We partitioned the dataset $D$ into $D_c$ and $D_f$ using a stratified sampling approach where $D_c$ contains 60\% of the data and $D_f$ contains 40\% of the data.

In this section, we present the complete set of results for RQ1 and RQ2: does \algname enforce DI constraints, with high probability, when existing fairness-aware algorithms often fail; and what is the cost of enforcing DI constraints.
In particular, we show the performance of the five algorithms being compared, in terms of failure rate, probability of returning a solution, and accuracy, for different values of $\alpha$ and as a function of $n$. Notice that Figures~\ref{fig:complete comparison pt1} and~\ref{fig:complete comparison pt2} present results consistent with the observations made in Section~\ref{sec: experiments}: the qualitative behavior of all considered algorithms remains the same for all values of $\alpha$.

\begin{figure}
    \centering
    \includegraphics[width=\columnwidth]{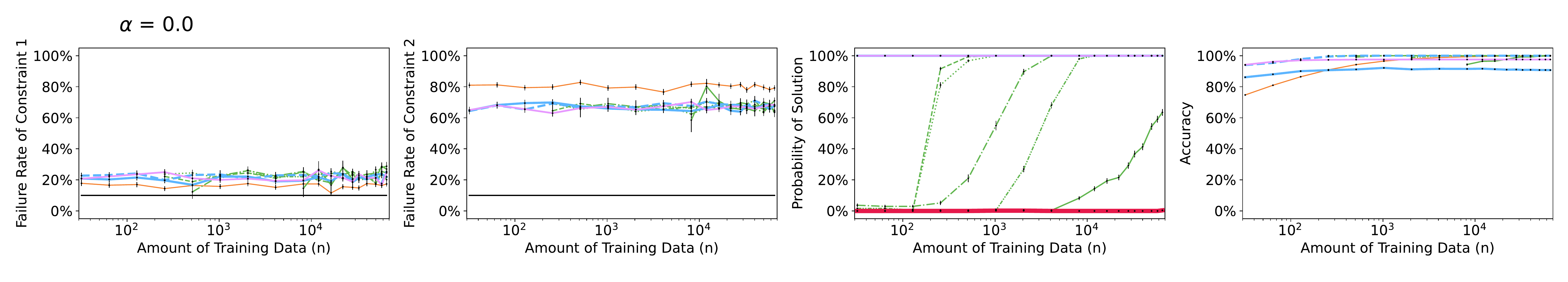}
    \includegraphics[width=\columnwidth]{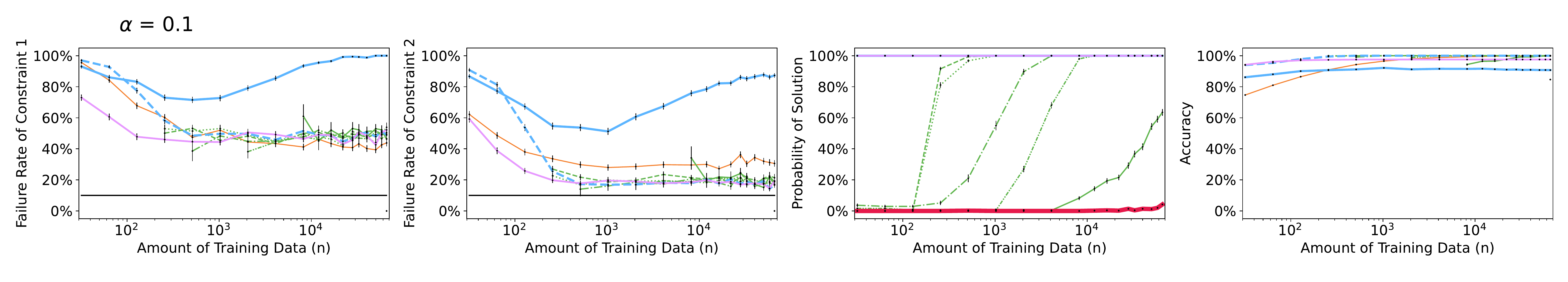}
    \includegraphics[width=\columnwidth]{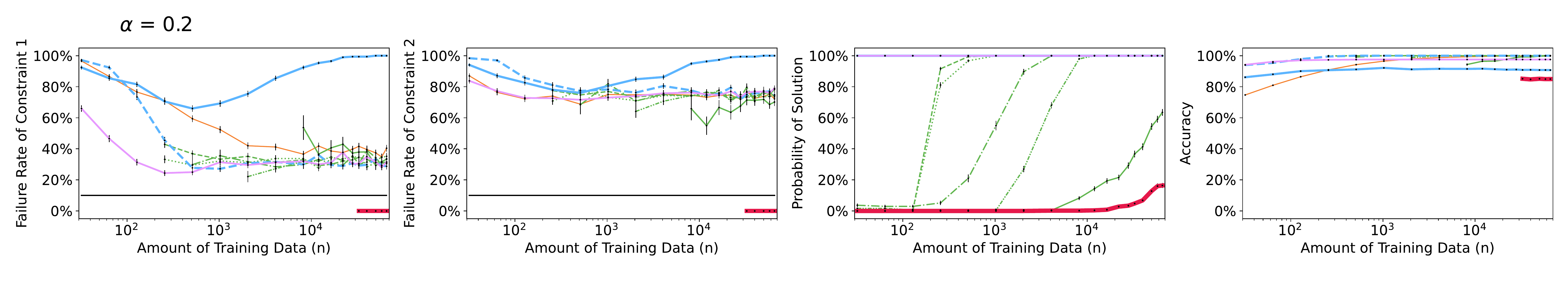}
    \includegraphics[width=\columnwidth]{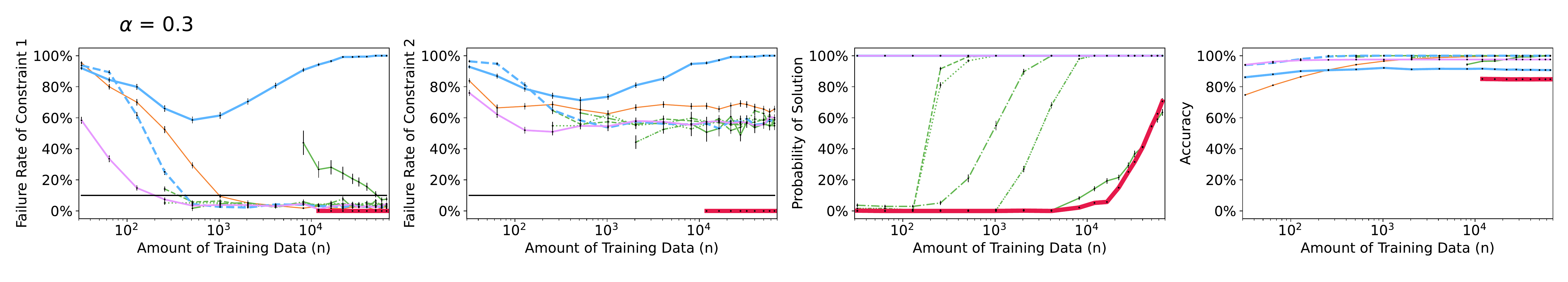}
    \includegraphics[width=\columnwidth]{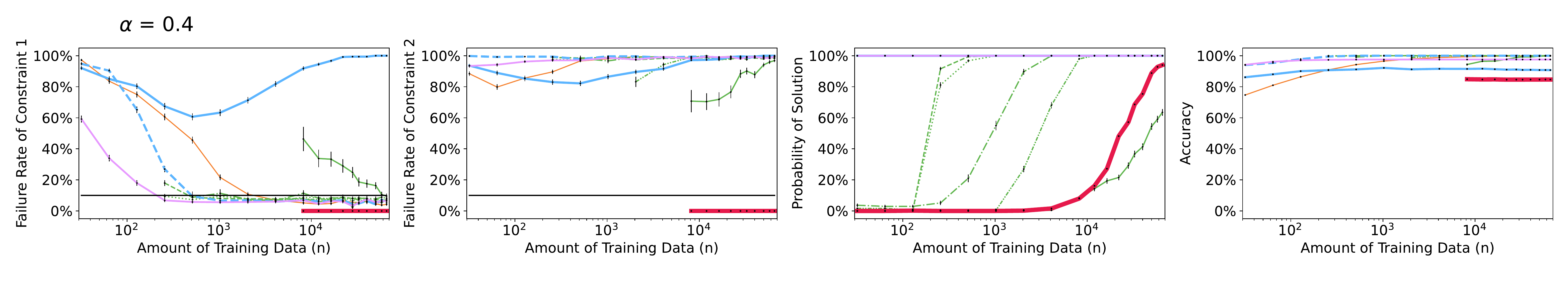}
    \caption{Algorithms' performances in terms of failure rate (leftmost two columns), probability of returning a solution (second from the right column), and accuracy (right column), as a function of $n$ and for different values of $\alpha$. The black horizontal lines indicate the maximum admissible probability of unfairness, $\delta_0=\delta_1=10\%$. All plots use the following legend: \elfline\,ELF~~ \lrline\,LR~~ \qsaDPline\,QSA with DP~~ \qsaEqOddsline\,QSA with EqOdds~~ \qsaEqOppline\,QSA with EqOpp~~ \qsaPEline\,QSA with PE~~ \qsaDisImpline\,QSA with DisImp~~ \flDPline\,Fairlearn with DP~~ \flEqOddsline\,Fairlearn with EqOdds~~ \fcline\,Fairness Constraints.}
    \label{fig:complete comparison pt1}
\end{figure}

\begin{figure}
    \centering
    \includegraphics[width=\columnwidth]{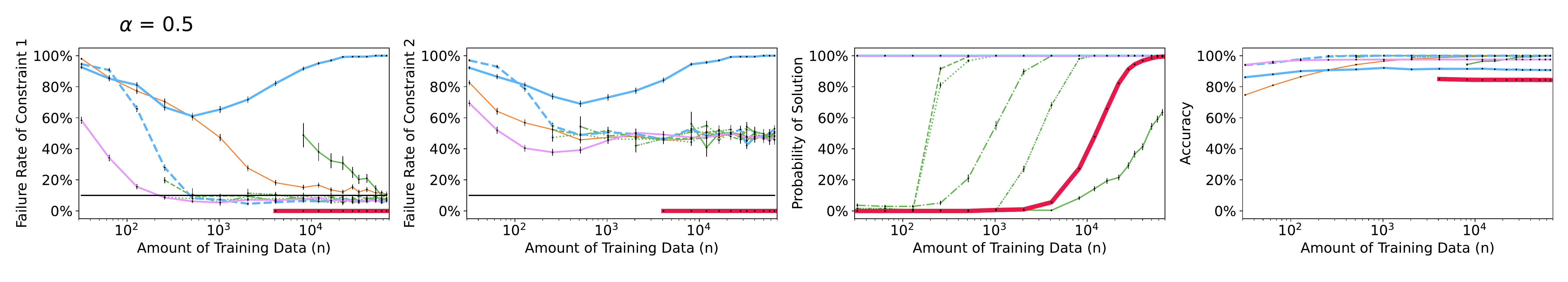}
    \includegraphics[width=\columnwidth]{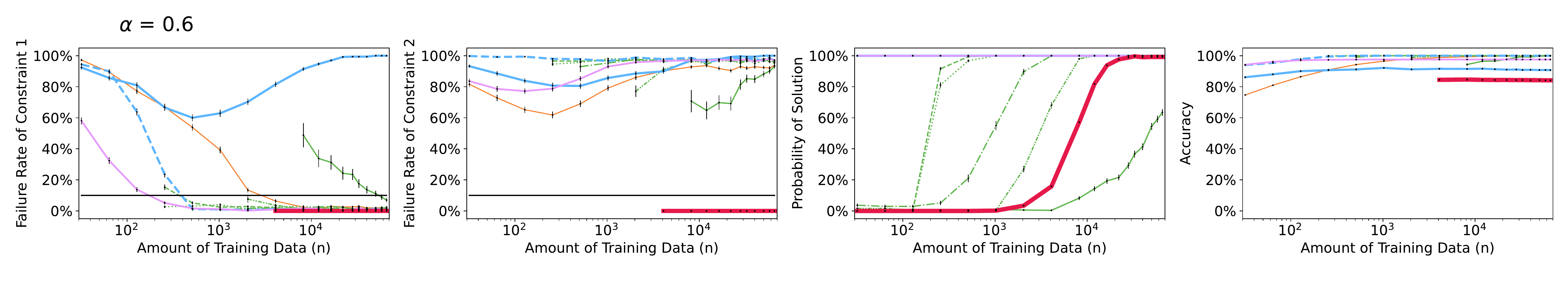}
    \includegraphics[width=\columnwidth]{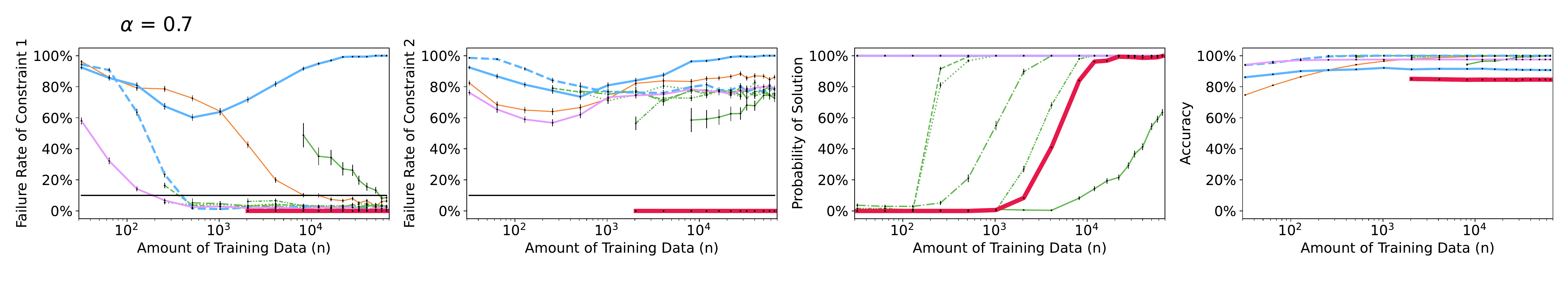}
    \includegraphics[width=\columnwidth]{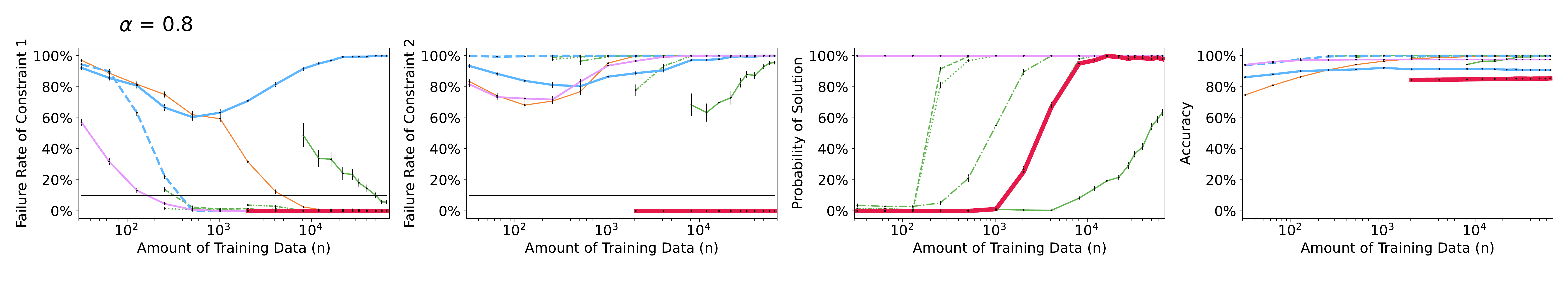}
    \includegraphics[width=\columnwidth]{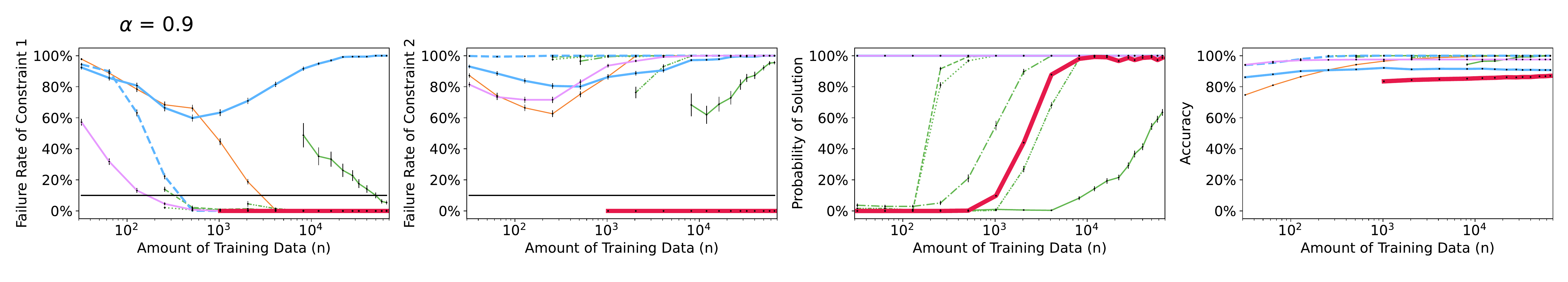}
    \includegraphics[width=\columnwidth]{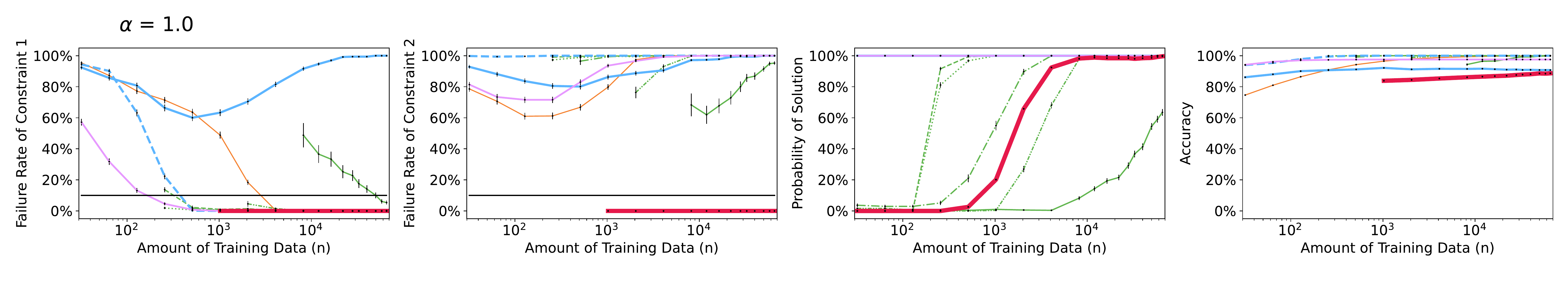}
    \caption{Algorithms' performance in terms of failure rate (leftmost two columns), probability of returning a solution (second from the right column), and accuracy (right column), as a function of $n$ and for different values of $\alpha$. The black horizontal lines indicate the maximum admissible probability of unfairness, $\delta_0=\delta_1=10\%$. All plots use the following legend: \elfline\,ELF~~ \lrline\,LR~~ \qsaDPline\,QSA with DP~~ \qsaEqOddsline\,QSA with EqOdds~~ \qsaEqOppline\,QSA with EqOpp~~ \qsaPEline\,QSA with PE~~ \qsaDisImpline\,QSA with DisImp~~ \flDPline\,Fairlearn with DP~~ \flEqOddsline\,Fairlearn with EqOdds~~ \fcline\,Fairness Constraints.}
    \label{fig:complete comparison pt2}
\end{figure}

\end{document}